\documentclass[12pt]{article}
\usepackage{arxiv}

\usepackage[utf8]{inputenc} %
\usepackage[T1]{fontenc}    %
\usepackage{url}            %
\usepackage{booktabs}       %
\usepackage{amsfonts}       %
\usepackage{nicefrac}       %
\usepackage{microtype}      %
\usepackage{xcolor}         %
\usepackage[inline]{enumitem} %
\usepackage[numbers,sort&compress]{natbib}
\usepackage{doi}            %
\usepackage{algorithm,algpseudocode}
\usepackage{amsmath}
\usepackage{amssymb}
\usepackage{mathtools}
\usepackage{amsthm}
\usepackage{multirow, array, threeparttable}
\usepackage[font=normalsize,skip=4pt]{caption}
\usepackage{subcaption}
\usepackage{graphicx}
\hypersetup{
  colorlinks,
  linkcolor={blue!60!black},
  citecolor={red!60!black},
  urlcolor={blue!80!black}
}
\usepackage[capitalize,nameinlink]{cleveref} %

\makeatletter
\renewcommand{\paragraph}{%
  \@startsection{paragraph}{4}{\z@}%
  {0.25\baselineskip}  %
  {-0.25em}            %
  {\normalfont\normalsize\bfseries}%
}
\makeatother

\usepackage{dsfont}

\algrenewcommand\algorithmiccomment[1]{\hfill // #1}
\algrenewcommand\algorithmicrequire{\textbf{Require:}}
\algrenewcommand\algorithmicensure{\textbf{Ensure:}}

\DeclareMathOperator{\softmax}{softmax}
\DeclareMathOperator{\diag}{diag}
\DeclareMathOperator{\Cov}{Cov}
\DeclareMathOperator{\Var}{Var}
\newcommand{\iid}{\overset{\text{iid}}{\sim}}

\theoremstyle{plain}
\newtheorem{proposition}{Proposition}[section]
\theoremstyle{definition}
\newtheorem{definition}{Definition}[section]

\title{Calibrating Scientific Foundation Models with Inference-Time Stochastic Attention}

\author{%
{\hspace{1mm}Akash Yadav$^{*}$} \\
	University of Houston\\
	\texttt{ayadav4@uh.edu} \\
        \And
{\hspace{1mm}Taiwo A. Adebiyi$^{*}$} \\
	University of Houston\\
	\texttt{taadebiyi2@uh.edu} \\
        \And
{\hspace{1mm}Ruda Zhang\textsuperscript{\textdagger}} \\
	University of Houston\\
	\texttt{rudaz@uh.edu} \\
}

\date{}                          %
\renewcommand{\headeright}{ }    %
\renewcommand{\shorttitle}{Calibrating Scientific Foundation Models with Inference-Time Stochastic Attention}

\begin{document}
\maketitle

\begingroup
\renewcommand\thefootnote{*}
\footnotetext{Equal contribution.}
\renewcommand\thefootnote{\textdagger}
\footnotetext{Corresponding author.}
\renewcommand\thefootnote{}
\footnotetext{Published at the 40th Conference on Neural Information Processing Systems (NeurIPS 2026).}
\endgroup

\begin{abstract}
  Transformer-based scientific foundation models are increasingly deployed in high-stakes settings,
  but current architectures give deterministic outputs and provide limited support
  for calibrated predictive uncertainty.
  We propose \textit{Stochastic Attention}, a lightweight sample-average inference-time modification
  that randomizes attention by replacing softmax weights with normalized multinomial samples
  controlled by a single concentration parameter, and produces predictive ensembles
  without retraining. To set this parameter, we introduce a calibration objective
  that matches the stochastic attention output with the target,
  yielding an efficient univariate post-hoc tuning problem.
  We evaluate this mechanism on scientific foundation models for weather and time-series forecasting, as well as a regression task. Across benchmarks against baseline methods, Stochastic Attention achieves the strongest native
  calibration and the sharpest prediction intervals at comparable calibration, with adaptation costs nearly three orders of magnitude lower than the next-best baseline. Code is available at \url{https://github.com/UQUH/StochasticAttention}.
\end{abstract}

\section{Introduction}

As deep learning enters scientific modeling, the field has had to confront an uncomfortable fact: some of its most accurate modern predictors are mechanistically opaque. Across scientific domains, these models are no longer viewed merely as flexible function approximators, but increasingly as useful complements to classical modeling pipelines. In weather forecasting, for example, learned systems now rival, and in some cases surpass, long-standing forecasting workflows on prominent benchmarks (Figure~\ref{fig:teaser}a) \cite{NguyenClimaX23}. What remains unresolved is whether such models can be trusted when uncertainty matters. Scientific predictors often operate under heterogeneous data, noisy supervision, misspecification, and distribution shift; in these regimes, a point forecast is insufficient without calibrated uncertainty to support decisions.

\begin{figure}[!t]
	\centering
    \includegraphics{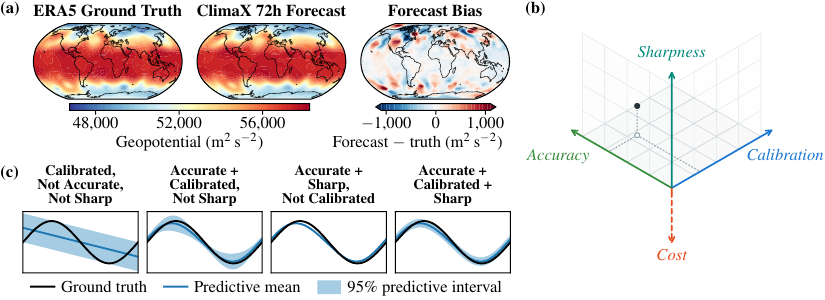}
	\caption{Useful uncertainty in scientific foundation models is multi-dimensional. \textbf{(a)} ClimaX 72-hour weather forecasting: ERA5 ground truth, model forecast, and bias. \textbf{(b)} Accuracy, calibration, sharpness, and cost as complementary axes of uncertainty quality. \textbf{(c)} Sinusoidal example: predictors recovering the same trend can attach qualitatively different uncertainty intervals to it.}
	\label{fig:teaser}
\end{figure}

This requirement is not met by simply attaching intervals to a deterministic prediction. Uncertainty quality has several dimensions that do not collapse into one another. A model can be accurate but overconfident, calibrated but too diffuse to support action, or sharp only because it understates risk. Figure~\ref{fig:teaser}b summarizes the four axes used throughout the paper: accuracy, calibration, sharpness, and cost. Figure~\ref{fig:teaser}c gives a simple example in which predictors recover the same sinusoidal trend while assigning qualitatively different uncertainty to the same inputs. This distinction is standard in probabilistic forecasting, where calibration and sharpness are complementary properties rather than interchangeable criteria \cite{Gneiting2007,Kuleshov2022}. For scientific foundation models, cost is a key lens: uncertainty methods that require repeated training, posterior approximation, or large ensembles can become impractical as model scale increases.

The existing uncertainty literature provides several ways to construct predictive distributions, but these approaches often trade off calibration, sharpness, or computational practicality depending on where they intervene. Output-level calibration methods, including temperature scaling, calibrated regression, and conformal procedures, adjust uncertainty estimates after predictions have already been produced \cite{Guo2017,Kuleshov2018,Vovk2005,Ye2024}. Weight-space and ensemble methods, including deep ensembles, SWAG, MultiSWAG, and recent variational approaches, place uncertainty over parameters, optimization trajectories, or independently trained models, often requiring repeated training, checkpoint collection, or modified optimization \cite{Lakshminarayanan2017,Maddox2019,Wilson2020,Shen2024}. Stochastic forward-pass methods, such as MC Dropout, generate samples from a single trained network, whereas attention-level stochastic methods intervene within the transformer's routing computation \cite{Gal2016,Heo2018,Pei2022}. These are not merely implementation choices: they determine what source of variability is represented, when calibration must be imposed, and what computational cost is paid.

We address this trade-off at the attention bottleneck. The motivation is direct: deterministic attention already carries a probabilistic structure. For each query, the softmax weights define a categorical distribution over tokens, and the standard attention output is the expectation under that distribution. We introduce Stochastic Attention (SA), a lightweight sample-average stochastic attention mechanism that replaces this exact expectation with a normalized multinomial sample average at inference time. The construction preserves the deterministic attention output in expectation, induces controllable dispersion through a single concentration parameter $\nu$, and requires no retraining of the backbone. We select $\nu$ using a calibration criterion that matches attention-induced stochastic variability to the residual scale observed on held-out data, thereby yielding an efficient one-dimensional post hoc calibration problem for a fixed deterministic predictor.
Our contributions are threefold:
\begin{itemize}
\item We introduce lightweight Stochastic Attention, an inference-time uncertainty mechanism that interprets deterministic attention as an expectation under token probabilities and replaces it with normalized multinomial sampling.
\item We characterize its mean-preserving behavior and $\nu$-controlled dispersion, and formulate a calibration-first objective that selects $\nu$ by matching attention-induced variability to held-out residual scale.
\item We evaluate SA across TimesFM forecasting on ETT datasets, ClimaX global weather forecasting, and eight UCI regression tasks with FT-Transformer, showing strong native calibration and sharp uncertainty at comparable calibration, often exceeding ensemble-style baselines at substantially lower computational cost.
\end{itemize}

\section{Uncertainty quality as an operating point}
\label{sec:lens}

In practice, useful uncertainty is an operating point, not a scalar score. Accuracy concerns the quality of a representative prediction; calibration, whether the predictive distribution is statistically consistent with realized outcomes; sharpness, the concentration of that distribution, and therefore how informative the reported uncertainty is once it is credible; cost, the computational burden required to obtain the estimate. Calibration and sharpness are especially easy to conflate, but the distinction is operational: calibration establishes credibility, and sharpness determines informativeness once credibility holds. Fig.~\ref{fig:teaser}b illustrates the non-collapse of these axes; Fig.~\ref{fig:teaser}c gives a controlled example where similar predictive means support qualitatively different uncertainty claims \cite{Gneiting2007,Kuleshov2022}.

Cost is part of the same operating point. The weight-space and ensemble methods above obtain predictive variability through additional model construction (repeated training, checkpoint collection, or modified optimization) \cite{Lakshminarayanan2017,Maddox2019,Wilson2020,Shen2024}; MC Dropout \cite{Gal2016} avoids repeated training but does not by itself determine a calibrated stochastic scale; attention-level methods such as Hierarchical Stochastic Attention (HSA) \cite{Heo2018,Pei2022} require architectural or training-stage changes. This distinction grows with model scale: when uncertainty quantification (UQ) scales through repeated model construction, reliable UQ becomes least accessible where large pretrained scientific models are most needed.

Proper scoring rules provide an important but different lens. Negative log-likelihood, the continuous ranked probability score (CRPS), and Energy Score reward predictive distributions that assign appropriate mass to realized outcomes \cite{Gneiting2007,Gruber2022}; we report them where appropriate. A single proper score, however, can improve through point accuracy, dispersion, tails, or a different calibration--sharpness trade-off, so we treat proper scores as operating-point diagnostics rather than replacements for explicit calibration and sharpness comparisons.

\section{Stochastic attention} \label{sec:method}

Stochastic attention uses the probability structure already computed by deterministic attention. Each attention row is a distribution over tokens, and the standard attention output is the corresponding expectation. SA turns this deterministic expectation into a sample-average stochastic computation, with a concentration parameter $\nu$ controlling how tightly the stochastic output concentrates around deterministic attention.

\subsection{Deterministic attention as an expectation}

Consider a single attention head with query, key, and value matrices
$\mathbf{Q} \in \mathbb{R}^{n_q \times d}$,
$\mathbf{K} \in \mathbb{R}^{n_k \times d}$, and
$\mathbf{V} \in \mathbb{R}^{n_k \times d_v}$.
For a query position $t \in \{1,\ldots,n_q\}$, let $\mathbf{q}_t\in\mathbb{R}^d$ denote the $t$-th query row and let $\mathbf{k}_j\in\mathbb{R}^d$ and $\mathbf{v}_j\in\mathbb{R}^{d_v}$ denote the $j$-th key and value rows, respectively.
The score vector $\mathbf{s}_t \in \mathbb{R}^{n_k}$ is defined componentwise by
\begin{equation}
(\mathbf{s}_t)_j
:=
\frac{\langle \mathbf{q}_t,\mathbf{k}_j\rangle}{\sqrt d},
\qquad
j=1,\ldots,n_k.
\end{equation}
The corresponding deterministic attention weights are
$\boldsymbol{\pi}_t
:=
\softmax(\mathbf{s}_t)
\in \Delta^{n_k-1}$,
where $\Delta^{n_k-1}$ is the probability simplex. We regard $\boldsymbol{\pi}_t$ as a column vector. The standard attention output at position $t$ is the row vector
\begin{equation}
\mathbf{o}_t
:=
\boldsymbol{\pi}_t^\intercal \mathbf{V}
=
\sum_{j=1}^{n_k}(\boldsymbol{\pi}_t)_j\mathbf{v}_j .
\label{eq:det_attn_output}
\end{equation}

This expectation view is the structural fact the method uses. The attention weights are not only normalized coefficients; they define token probabilities, and $\mathbf{o}_t$ is the expected value vector under those probabilities. Standard masks, including causal and padding masks, are applied before the softmax as usual. SA acts only after the masked softmax has produced a valid simplex-valued attention vector.

\subsection{Stochastic attention via normalized multinomial sampling}

With this expectation view fixed, SA replaces exact averaging under $\boldsymbol{\pi}_t$ with a normalized multinomial empirical average. The concentration parameter $\nu$ controls the sample size of that approximation.

\begin{definition}[Stochastic attention weights] \label{def:stochastic_attention_weights}
Fix an integer $\nu \in \mathbb{N}$. Given $\boldsymbol{\pi}_t \in \Delta^{n_k-1}$, draw counts
$\mathbf{w}_t \sim \operatorname{Multinomial}(\nu,\boldsymbol{\pi}_t)$,
and define the stochastic attention weights
$\widetilde{\boldsymbol{\pi}}_t
:=
\frac{1}{\nu}\mathbf{w}_t
\in
\Delta^{n_k-1}$.

Equivalently, if $z_t^{(m)} \iid \operatorname{Categorical}(\boldsymbol{\pi}_t),
m=1,\ldots,\nu$,
and $\mathbf{e}_j$ denotes the $j$-th standard basis vector, then
$\widetilde{\boldsymbol{\pi}}_t
=
\frac{1}{\nu}
\sum_{m=1}^{\nu}
\mathbf{e}_{z_t^{(m)}}$.
The corresponding %
output is
$\widetilde{\mathbf{o}}_t
:=
\widetilde{\boldsymbol{\pi}}_t^\intercal \mathbf{V}
=
\frac{1}{\nu}
\sum_{m=1}^{\nu}
\mathbf{v}_{z_t^{(m)}}$.
\end{definition}

Thus, stochastic attention replaces the deterministic expectation under $\boldsymbol{\pi}_t$ by a finite-sample empirical average. If one token is attended to at random according to $\boldsymbol{\pi}_t$, the attention outcome is random. If $\nu$ such token selections are averaged, the result remains random but concentrates toward deterministic attention as $\nu$ grows. This is the stochastic family we study, with $\nu$ acting as a concentration parameter.
The first basic property is mean preservation and deterministic recovery.

\begin{proposition}[Mean preservation and deterministic recovery] \label{prop:unbiased}
Conditioned on $(\boldsymbol{\pi}_t,\mathbf{V})$, the stochastic weights and outputs satisfy
\begin{equation}
\mathbb{E}[\widetilde{\boldsymbol{\pi}}_t \mid \boldsymbol{\pi}_t]
=
\boldsymbol{\pi}_t,
\qquad
\mathbb{E}[\widetilde{\mathbf{o}}_t \mid \boldsymbol{\pi}_t,\mathbf{V}]
=
\mathbf{o}_t.
\end{equation}
Moreover, under the i.i.d. categorical representation,
\begin{equation}
\widetilde{\boldsymbol{\pi}}_t \to \boldsymbol{\pi}_t
\quad\text{and}\quad
\widetilde{\mathbf{o}}_t \to \mathbf{o}_t
\quad\text{almost surely as } \nu\to\infty.
\end{equation}
\end{proposition}

Proposition~\ref{prop:unbiased} concerns one attention row, conditional on its inputs. In a deep network the perturbed output of one layer feeds the nonlinear layers that follow, so the end-to-end predictive mean of the stochastic network is not exactly the deterministic prediction. The departure is controlled by $\nu$: the per-row perturbation shrinks at the rate of Proposition~\ref{prop:variance} and vanishes in the deterministic limit. It is also not left unmodeled, because $\nu$ is calibrated on the end-to-end residuals of the stochastic network itself (Section~\ref{sec:calibration}); every metric we report is a property of the perturbed process, and the deterministic prediction remains available at no extra cost.

The second property describes how the induced dispersion scales with $\nu$.

\begin{proposition}[Covariance scaling with $\nu$] \label{prop:variance}
Conditioned on $\boldsymbol{\pi}_t$, the stochastic attention weights satisfy
\begin{equation}
\Cov(\widetilde{\boldsymbol{\pi}}_t \mid \boldsymbol{\pi}_t)
=
\frac{1}{\nu}
\left(
\diag(\boldsymbol{\pi}_t)
-
\boldsymbol{\pi}_t\boldsymbol{\pi}_t^\intercal
\right).
\end{equation}
Consequently,
\begin{equation}
\Cov(\widetilde{\mathbf{o}}_t \mid \boldsymbol{\pi}_t,\mathbf{V})
=
\mathbf{V}^\intercal
\Cov(\widetilde{\boldsymbol{\pi}}_t \mid \boldsymbol{\pi}_t)
\mathbf{V} 
=
\frac{1}{\nu}
\mathbf{V}^\intercal
\left(
\diag(\boldsymbol{\pi}_t)
-
\boldsymbol{\pi}_t\boldsymbol{\pi}_t^\intercal
\right)
\mathbf{V}.
\end{equation}
\end{proposition}

These two results make $\nu$ the natural operating parameter, and they show what it controls and what it leaves to the trained model. At each attention row, the stochastic output equals the deterministic one on average (Proposition~\ref{prop:unbiased}), and its covariance is a \emph{scale}, $1/\nu$, times a \emph{shape}, $\mathbf{V}^\intercal(\diag(\boldsymbol{\pi}_t)-\boldsymbol{\pi}_t\boldsymbol{\pi}_t^\intercal)\mathbf{V}$ (Proposition~\ref{prop:variance}). The shape equals $\sum_j(\boldsymbol{\pi}_t)_j(\mathbf{v}_j-\mathbf{o}_t)^\intercal(\mathbf{v}_j-\mathbf{o}_t)$, the covariance of one sampled value vector $\mathbf{v}_{z_t^{(m)}}$, and averaging the $\nu$ samples divides it by $\nu$. It is therefore large when attention divides its weight among tokens whose value vectors differ, and small when the weight is placed on a single token or on tokens with similar value vectors. Because the attention weights and value vectors change with the input, so does the dispersion, in every layer, head, and attention row, and none of this is designed or tuned. The single number $\nu$ sets only its overall size, which the calibration of Section~\ref{sec:calibration} fits to held-out residuals. SA thus produces its own input-dependent spread, and $\nu$ only scales it, whereas a global temperature can only rescale a spread already produced by another method, and a deterministic model produces none. SA is a probabilistic method but not a Bayesian one, since it places no prior distribution over the network's parameters; it can instead be interpreted as a frequentist approach.

This construction differs from single-draw hard attention or discrete stochastic attention schemes \cite{Heo2018}. Randomness is not locked to one sampled token or route. Normalized multinomial sampling gives a tunable empirical approximation to deterministic attention: dispersion can be increased or reduced at fixed learned attention scores by changing $\nu$. It also differs from dropout-style stochasticity, where randomness is imposed through independent masking rather than sampled from the token distribution computed by attention.
Algorithm~\ref{alg:stochastic_attention} summarizes the row-wise stochastic-attention computation in direct output form.

\begin{algorithm}[t]
\caption{Stochastic attention for one attention row}
\label{alg:stochastic_attention}
\small
\begin{algorithmic}[1]
\Require Deterministic attention weights $\boldsymbol{\pi}_t\in\Delta^{n_k-1}$; value matrix $\mathbf{V}\in\mathbb{R}^{n_k\times d_v}$; concentration parameter $\nu\in\mathbb{N}$.
\Ensure Stochastic attention output $\widetilde{\mathbf{o}}_t\in\mathbb{R}^{1\times d_v}$.

\For{$m=1$ \textbf{to} $\nu$}
    \State Sample $z_t^{(m)} \sim \operatorname{Categorical}(\boldsymbol{\pi}_t)$
\EndFor
\State Compute $\widetilde{\mathbf{o}}_t
=
\frac{1}{\nu}\sum_{m=1}^{\nu}\mathbf{v}_{z_t^{(m)}}$.
\State \Return $\widetilde{\mathbf{o}}_t$
\end{algorithmic}
\end{algorithm}

\subsection{Predictive distribution from repeated forward passes}

Let $f_{\boldsymbol{\theta}}$ denote the deterministic transformer and $f_{\boldsymbol{\theta},\nu}(\cdot)$ the same network with stochastic attention enabled. Given input $\mathbf{x}$, repeated stochastic forward passes generate Monte Carlo predictions
\begin{equation}
\widehat{\mathbf{y}}^{(s)}(\mathbf{x})
:=
f_{\boldsymbol{\theta},\nu}^{(s)}(\mathbf{x}),
\qquad
s=1,\ldots,S,
\label{eq:mc_predictions}
\end{equation}
which we treat as samples from the implicit predictive distribution induced by attention-space sampling in the frozen network. For scalar targets, the samples define an empirical predictive CDF $\widehat{F}_{\nu,\mathbf{x}}$; for vector-valued outputs, a sample-based joint distribution summarized coordinate-wise or through empirical covariance and prediction intervals.

The deterministic predictor $f_{\boldsymbol{\theta}}$ remains the reference around which SA's variability is centered. SA thus builds predictive variability from the internal attention mechanism without retraining the network or sampling a posterior over all weights. Two sampling quantities play distinct roles: $\nu$ defines the stochastic predictor by setting attention-space variability per pass, while $S$ controls the Monte Carlo accuracy with which summaries of the resulting predictive distribution are estimated. $\nu$ is therefore the method's operating parameter; $S$ is the evaluation budget.
Taken together, SA defines a one-parameter family of stochastic predictors $\{f_{\boldsymbol{\theta},\nu}\}_{\nu\in\Xi}$ around the fixed deterministic predictor $f_{\boldsymbol{\theta}}$, where $\Xi\subset\mathbb{N}$ is the candidate set for $\nu$.

\section{Choosing a calibrated operating point}
\label{sec:calibration}

Having defined the stochastic family $\{f_{\boldsymbol{\theta},\nu}\}_{\nu\in\Xi}$, we select its operating value by matching attention-induced stochastic variability to held-out residual scale. This makes $\nu$ a calibrated concentration parameter rather than a heuristic noise setting.

\subsection{Calibration objective}

We choose $\nu$ using a held-out calibration set
$\mathcal{D}_{\mathrm{cal}}
=
\{(\mathbf{x}_i, \mathbf{y}_i)\}_{i=1}^{N_{\mathrm{cal}}}$.
Let $f_{\boldsymbol{\theta}}(\mathbf{x})$ denote the deterministic predictor, and let $f_{\boldsymbol{\theta},\nu}(\mathbf{x})$ denote the random predictive output produced by one stochastic forward pass of the same frozen model under stochastic attention. Expectations and variances involving $f_{\boldsymbol{\theta},\nu}(\mathbf{x})$ are taken with respect to the internal stochastic-attention sampling.

For a calibration pair $(\mathbf{x}, \mathbf{y})$, define the induced stochastic deviation magnitude and deterministic residual magnitude as
\begin{equation}
R_\nu(\mathbf{x})
:=
\left\|
f_{\boldsymbol{\theta},\nu}(\mathbf{x})
-
f_{\boldsymbol{\theta}}(\mathbf{x})
\right\|_2,
\qquad
Z(\mathbf{x},\mathbf{y})
:=
\left\|
\mathbf{y}
-
f_{\boldsymbol{\theta}}(\mathbf{x})
\right\|_2.
\label{eq:z_xy}
\end{equation}
We then choose $\nu$ by minimizing the squared discrepancy between these two magnitudes:
\begin{equation}
\nu^\star
\in
\arg\min_{\nu\in\Xi}
\mathcal{L}_{\mathrm{SA}}(\nu),
\quad\text{where}\quad
\mathcal{L}_{\mathrm{SA}}(\nu)
:=
\mathbb{E}_{(\mathbf{x},\mathbf{y})\sim\mathcal{D}_{\mathrm{cal}}}
\left[
\mathbb{E}
\left[
\left(R_\nu(\mathbf{x})-Z(\mathbf{x},\mathbf{y})\right)^2
\,\middle|\,
\mathbf{x},\mathbf{y}
\right]
\right].
\label{eq:lsa}
\end{equation}
The objective compares two quantities in the same output space: the stochastic deviation induced by SA from the deterministic prediction, and the realized residual of that deterministic prediction. It therefore calibrates the scale of attention-induced predictive variability against held-out prediction error.

\subsection{Interpretation}

Because $Z(\mathbf{x},\mathbf{y})$ is fixed once the calibration pair $(\mathbf{x},\mathbf{y})$ is fixed, the inner expectation in $\mathcal{L}_{\mathrm{SA}}(\nu)$ decomposes as
\begin{equation}
\mathcal{L}_{\mathrm{SA}}(\nu)
=
\mathbb{E}_{(\mathbf{x},\mathbf{y})\sim\mathcal{D}_{\mathrm{cal}}}
\left[
\Var\!\left(R_\nu(\mathbf{x})\mid \mathbf{x},\mathbf{y}\right)
+
\left(
\mathbb{E}\!\left[R_\nu(\mathbf{x})\mid \mathbf{x},\mathbf{y}\right]
-
Z(\mathbf{x},\mathbf{y})
\right)^2
\right].
\label{eq:lsa_bias_variance}
\end{equation}
This decomposition gives the objective a direct meaning in calibration. The squared-bias term matches the expected magnitude of stochastic-attention deviation to the residual magnitude observed in held-out data. The variance term penalizes unnecessary variability in that induced deviation. Thus, $\mathcal{L}_{\mathrm{SA}}$ is not a generic noise-tuning rule; it is a dispersion-matching criterion for the predictive variability generated by stochastic attention. The selected $\nu^\star$ is therefore the operating value whose induced stochastic deviations match the held-out residual scale.

\subsection{Optimization and computational role of $\nu$}

In practice, $\mathcal{L}_{\mathrm{SA}}(\nu)$ is evaluated on the held-out calibration split $\mathcal{D}_{\mathrm{cal}}$. For each candidate $\nu\in\Xi$, we estimate the inner expectation using repeated stochastic forward passes of the frozen model.
When all SA layers share a single concentration parameter, $\mathcal{L}_{\mathrm{SA}}(\nu)$ becomes a one-dimensional calibration problem over $\Xi$. We optimize it using Bayesian optimization (BO) under uncertainty, following the routine of \citet{Yadav2025sobo}, which fits a Bayesian generalized linear model surrogate to noisy evaluations and proposes the next candidate $\nu$. Full details are deferred to the appendix. In practice, each BO iteration evaluates a small batch of proposed values of $\nu$ on a fresh mini-batch from $\mathcal{D}_{\mathrm{cal}}$, with one stochastic forward pass per example. Combined with the one-dimensional search and analytic proposal step, this makes $\nu$-selection fast in practice and far cheaper than retraining-heavy uncertainty baselines.

\section{Experimental design}
\label{sec:experiments}

\subsection{Models, tasks, and baselines}

We evaluate SA in three transformer-based prediction settings. TimesFM on the ETT-small benchmarks \cite{Das2024,Zhou2021} is a diagnostic time-series forecasting setting in a difficult low-data adaptation regime. ClimaX \cite{NguyenClimaX23} is a structured scientific forecasting case: 72-hour global $Z_{500}$ forecasts on gridded ERA5 fields. FT-Transformer \cite{Gorishniy2021} on eight UCI scientific and engineering regression tasks (Concrete, Energy, Naval, Yacht, Protein, Wine, Kin8nm, Power Plant) under the standard Gal-style protocol \cite{Gal2016} provides an independent transformer-regression validation. Architecture, adaptation, split, and implementation details are in Appendix~\ref{app:exp-details}.

We compare against baselines representing different sources of predictive variability. SWAG and IVON provide weight-space uncertainty through checkpoint-based posterior approximation and variational optimization, respectively \cite{Maddox2019,Shen2024}; MultiSWAG, built from ten independently trained SWAG models, is the suite's representative of deep ensembles \cite{Lakshminarayanan2017,Wilson2020}. MC Dropout \cite{Gal2016} and Contextual Dropout \cite{Fan2021} provide stochastic forward-pass baselines, and Hierarchical Stochastic Attention (HSA) \cite{Pei2022} is the closest attention-level baseline since it also modifies attention, but through a training-stage stochastic architecture. Implementation details are in Appendix~\ref{app:baselines}.
Every calibration parameter, SA's $\nu$ and each baseline's temperature, is fitted on held-out data unused in fine-tuning: the validation year 2016 on ClimaX (training 1979--2015, test 2017--2018), with $R_\nu$ and $Z$ in Eq.~\eqref{eq:z_xy} computed on the 72-hour $Z_{500}$ field; the ETT validation segment on TimesFM; and each split's validation fold on UCI.

\subsection{Metrics, diagnostics, and cost}

We evaluate four quantities: accuracy, calibration, sharpness, and cost. Accuracy uses the task-standard metric (MAE for TimesFM, latitude-weighted anomaly correlation for ClimaX, normalized RMSE for UCI). Calibration is assessed through the probability integral transform (PIT), empirical coverage, and the Wasserstein-1 distance $W_1$ between the PIT distribution and the uniform distribution on $[0,1]$. PIT is computed from the empirical predictive CDF for scalar targets and then aggregated coordinate-wise across spatial locations for structured ClimaX outputs. With $S$ samples, the PIT of a forecast is set by the rank $k\in\{0,\ldots,S\}$ of the outcome among them, so we compute $W_1$ from the rank histogram, $W_1=\frac{1}{S+1}\sum_{b=1}^{S+1}\bigl|\widehat{G}_b-\frac{b}{S+1}\bigr|$, where $\widehat{G}_b$ is the fraction of ranks below $b$; it is zero for a calibrated ensemble.
On ClimaX every metric is area-weighted: each grid point counts in proportion to the cosine of its latitude, as in the latitude-weighted ACC.
Under exact calibration, PIT values are uniform, with U-shaped histograms indicating under-dispersion and hump-shaped histograms indicating over-dispersion \cite{Dawid1984,Gneiting2007}.

Sharpness is summarized through the mean prediction-interval width (MPIW); because it is only meaningful at comparable calibration \cite{Gneiting2007,Kuleshov2022}, we report native calibration first and, when methods differ substantially in coverage, apply scalar post-scaling on the calibration set via temperature scaling \cite{Guo2017} before comparing widths.
We report every baseline both natively and after temperature scaling, so that temperature scaling is compared as a method with a regime of its own rather than only as a correction; SA is never rescaled. Temperature scaling replaces each predictive sample $\widehat{\mathbf{y}}^{(s)}$ by $\bar{\mathbf{y}}+T(\widehat{\mathbf{y}}^{(s)}-\bar{\mathbf{y}})$, where $\bar{\mathbf{y}}$ is the mean of the $S$ samples and the scalar $T$ is fitted on the calibration set ($T=1$ is the native prediction).
Temperature scaling is the one post-hoc corrector we apply, uniformly to every baseline; split conformal prediction \cite{Vovk2005} is likewise an output-level step on top of a predictive model rather than a source of predictive uncertainty, so we treat it as complementary. Cost is reported as the wall-clock compute required to obtain each method's predictive distribution, since retraining-heavy and inference-time mechanisms do not occupy the same deployment regime even on identical tasks. Setting-specific computational details are in Appendix~\ref{app:exp-details}.

\section{Results}
\label{sec:results_forecast}

\subsection{TimesFM: diagnostic evidence across ETT}

\begin{figure}[!t]
\centering
\includegraphics[width=0.85\linewidth]{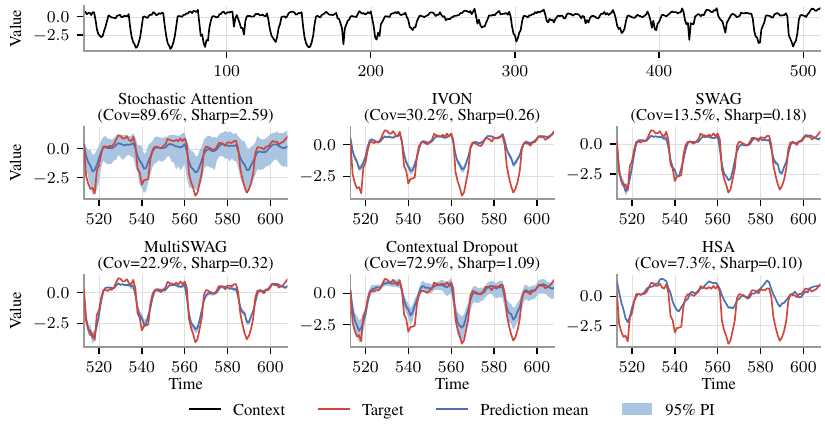}
\caption{Native uncertainty on TimesFM (ETTh1, horizon $H{=}96$, series HUFL, $S{=}100$ samples) on a test window. SA achieved empirical coverage of 89.6\%, closer to the 95\% target than any of the baselines (IVON 30.2\%, SWAG 13.5\%, MultiSWAG 22.9\%, Contextual Dropout 72.9\%, HSA 7.3\%). SA is closest to the 95\% target across all 665 test windows in this configuration (mean coverage: 87.8\%). Full results across all eight ETT configurations are in Appendix~\ref{app:timesfm-details}.}
\label{fig:timesfm_diagnostic}
\end{figure}

Figure~\ref{fig:timesfm_diagnostic} shows the native 95\% intervals of all six methods on one ETTh1 test window ($H{=}96$). SA's interval covers 89.6\% of the window, while the baselines' intervals cover 7--30\% (Contextual Dropout 73\%): narrow intervals that miss most realized outcomes are not sharp but overconfident. The pattern is not specific to this window. Overall, across 665 test windows of this configuration, SA covers 87.8\% of outcomes natively, between 85\% and 90\% on each of the seven series, and is closest to the 95\% target on every window, while the baselines cover 4--34\%. Coverage alone does not establish calibration, but a pattern this stable is the diagnostic that carries forward.

Temperature scaling, fitted per series on the validation segment, is the inexpensive correction (Table~\ref{tab:timesfm}, Appendix~\ref{app:timesfm-details}). It brings the baselines to 93--96\% coverage, but only by widening their intervals to $1.4$--$2.6\times$ SA's. SA covers 87.8\% natively and maintains the lowest $W_1$ ($0.029$, compared with $0.041$--$0.088$), an MAE second only to IVON's, and the lowest preparation cost (3 minutes, compared with 5--47). A temperature thus buys the last points of coverage with much wider intervals, whereas SA, without any correction, has the lower $W_1$ and the narrower intervals. This comparison holds across all eight ETT configurations (Table~\ref{tab:timesfm}): SA has the lowest $W_1$ on six (IVON and HSA on one each), the lowest mean ($0.035$; IVON $0.045$) and worst case ($0.047$; Contextual Dropout $0.070$), and the narrowest intervals on all eight. ClimaX tests whether this holds at scale, with a stronger backbone and high-dimensional field outputs.

\subsection{ClimaX: calibration, sharpness, and cost over 20 resampled splits}

\begin{table}[t]
\centering
\caption{ClimaX ($Z_{500}$, 72-hour lead; 17{,}376 test forecasts, $S{=}19$), mean\,$\pm$\,std over 20 resampled splits (8 for MultiSWAG). $W_1$ is reported natively and after temperature scaling (TS); SA is native. ACC is that of the 19-sample mean; $^{\dagger}$SA's deterministic prediction has ACC $0.9683\pm0.0006$. MPIW (width of the $[\min,\max]$ range) and spread (ensemble standard deviation) are in $\mathrm{m^2s^{-2}}$ at the operating points. Every metric is area-weighted by the cosine of latitude. Cost is the one-time preparation on one GPU; TS adds $1.3$\,h. Bold: best; underlined: second best.}
\label{tab:climax_summary}
\footnotesize
\setlength{\tabcolsep}{3pt}
\begin{tabular}{@{}lcccccc@{}}
\toprule
& & \multicolumn{2}{c}{$W_1\downarrow$} & \multicolumn{2}{c}{Sharpness$\downarrow$} & \\
\cmidrule(lr){3-4}\cmidrule(lr){5-6}
Method & ACC$\uparrow$ & native & TS & MPIW & spread & Cost$\downarrow$ \\
\midrule
SWAG & \underline{0.9691$\pm$0.0001} & 0.2253$\pm$0.0007 & 0.0681$\pm$0.0056 & 1040$\pm$54 & 393$\pm$19 & 28.6\,h \\
MultiSWAG & \textbf{0.9692$\pm$0.0001} & 0.2244$\pm$0.0008 & \underline{0.0678$\pm$0.0065} & 972$\pm$27 & 331$\pm$14 & 11.9\,d \\
IVON & 0.9388$\pm$0.0004 & 0.2377$\pm$0.0034 & 0.1059$\pm$0.0083 & 1409$\pm$76 & 483$\pm$35 & \underline{14.1\,h} \\
Contextual Dropout & 0.9662$\pm$0.0001 & \underline{0.1147$\pm$0.0040} & 0.0687$\pm$0.0051 & \underline{932$\pm$18} & \underline{314$\pm$6} & 34.7\,h \\
HSA & 0.7837$\pm$0.0038 & 0.1861$\pm$0.0018 & \textbf{0.0406$\pm$0.0034} & 2453$\pm$61 & 936$\pm$27 & 33.0\,h \\
\midrule
SA (ours) & 0.9570$\pm$0.0035$^{\dagger}$ & \textbf{0.0468$\pm$0.0056} & \,--\, & \textbf{672$\pm$40} & \textbf{218$\pm$12} & \textbf{3\,min} \\
\bottomrule
\end{tabular}
\end{table}

\begin{figure}[t]
\centering
\includegraphics{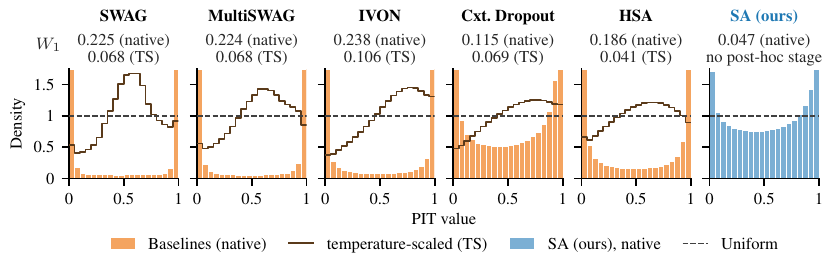}
\caption{Area-weighted PIT on ClimaX, pooled over the 20 resampled splits (8 for MultiSWAG), on one density axis (uniform $=1$, dashed). Bars are native; the dark step is the baseline after temperature scaling (TS). Above each panel are $W_1$ native and after TS. Native bins that run off the axis reach densities $2.3$--$12.4$ for the baselines and $2.2$ for SA's last bin. Natively the baselines are U-shaped (under-dispersed) and SA is near uniform; TS turns the U into a hump.}
\label{fig:climax_pit}
\end{figure}

Table~\ref{tab:climax_summary} summarizes 20 complete repetitions of the ClimaX pipeline, each re-fine-tuning the backbone with its own seed on a split whose calibration year is drawn from 1979--2016 (Appendix~\ref{app:climax-stats}). The under-dispersion seen on TimesFM persists at this scale and for all five baselines: natively their PIT histograms are U-shaped (Figure~\ref{fig:climax_pit}), with $W_1$ of $0.11$--$0.24$, while SA's $W_1$ is $0.047\pm0.006$. Temperature scaling brings the baselines to $W_1=0.041$--$0.106$, and SA's native $W_1$ is still statistically significantly lower than the scaled $W_1$ of SWAG, MultiSWAG, IVON and Contextual Dropout (paired $t$-test $p\le1.2\times10^{-4}$; Wilcoxon $p\le0.008$, the smallest value attainable with MultiSWAG's 8 runs); only scaled HSA is lower, by $0.006$. SA's native calibration thus exceeds that of four of the five baselines after temperature scaling.

At these calibration levels, SA's intervals are the narrowest in every paired run, by $260$--$1781\,\mathrm{m^2s^{-2}}$ on average (paired $t$-test $p<10^{-6}$ against every baseline), and its RMS spread is about $70\%$ of the next smallest. The two approaches miss the nominal level from opposite sides: SA's $[\min,\max]$ range covers $80\%$ of outcomes at the nominal $90\%$, mildly under-dispersed, while the scaled baselines cover $92$--$93\%$, over-dispersed. This is consistent with the shape-and-scale argument of Section~\ref{sec:method}: a temperature widens every forecast by the same factor, while SA's spread varies with the input.

HSA is the instructive exception. Like SA, it injects stochasticity at the attention bottleneck, but does so during training, swapping the pretrained softmax for an operator new to the backbone. As a result, its sample-mean ACC on ClimaX degrades to $0.784$, 19\% below the deterministic backbone, and it reaches the lowest scaled $W_1$ only with intervals $3.7\times$ as wide as SA's. SA avoids this failure mode: mean preservation (Proposition~\ref{prop:unbiased}) keeps each sampled attention row centered on the pretrained one. Its deterministic prediction is unchanged (ACC $0.9683\pm0.0006$, within $0.001$ of the best ensemble baseline), and the mean of its 19 samples trades about one point of ACC for calibration ($0.9570\pm0.0035$), rising toward the deterministic value as $\nu$ grows (Appendix~\ref{app:nu-ablation}).

Cost separates the methods by orders of magnitude (last column of Table~\ref{tab:climax_summary}). Even before temperature scaling, SWAG requires $28.6$ hours of trajectory collection, and MultiSWAG, which repeats it for ten independently trained models, requires $11.9$ days, whereas SA's BO calibration takes $3$ minutes on the frozen backbone. At inference, one SA forward pass takes $0.457$ seconds at $\nu{=}1$ and $0.471$ seconds at $\nu{=}16$, against $0.424$ seconds for the deterministic pass, so the overhead is at most $11\%$ over the calibrated range $\nu^\star\in[7,12]$; the $S$ passes are independent and batchable (Appendix~\ref{app:latency}).
At ClimaX scale (108 million parameters), retraining-heavy uncertainty quantification already costs hours to days of compute; for billion-parameter backbones it becomes inaccessible. What remains feasible at that frontier is inference-time sampling on a frozen backbone, the regime SA is built for.

To test a regime absent from calibration, we hold one season out of the calibration year and score that season of 2017--18 (Appendix~\ref{app:climax-shift}). Without their correction, the baselines remain miscalibrated (native $W_1$ $0.11$--$0.24$), while SA needs none and stays the sharpest with either season held out (MPIW $682$ and $645\,\mathrm{m^2s^{-2}}$, against $941$ and $991$ for the next narrowest). The shift does not, however, favor SA's calibration over a refit temperature. The baselines' temperatures change by at most $1.4\%$ when a season is removed, whereas holding winter out raises SA's $\nu^\star$ from 9 to 12, which narrows its intervals but raises its $W_1$ from $0.041$ to $0.053$.

\subsection{UCI regression with FT-Transformer}
\label{sec:results_uci}

We evaluate SA on eight UCI regression datasets using a Feature Tokenizer Transformer \cite{Gorishniy2021} trained from scratch on each dataset, in contrast to the fine-tuning setting of TimesFM and ClimaX. We follow the standard Gal-style protocol \cite{Gal2016}: 20 random 90/10 splits per dataset, 5 for Protein, with all six methods evaluated on the same splits and with $S{=}99$ predictive samples. To account for the small test set in each split, we report $W_1$ pooled over each dataset's splits. Full hyperparameters, split details, and the testing protocol are in Appendix~\ref{app:uci-protocol}.

\begin{table}[t]
\centering
\caption{UCI regression on eight datasets ($S{=}99$; 20 splits, 5 for Protein), sorted by accuracy. \textbf{Acc.}: RMSE$/\sigma_y$ of SA's sample mean. \textbf{Panel~A}: native pooled $W_1$, 95\% bootstrap interval beneath. \textbf{Panel~B}: mean width of the 90\% interval relative to SA's, baselines temperature-scaled and SA native. In Acc.\ and Panel~B, the standard deviation over splits is beneath each value. MC: MC Dropout, SW: SWAG, MS: MultiSWAG, IV: IVON. Lower is better; bold: best, underlined: second best per row.}
\label{tab:uci_summary}
\def\sm#1{{\tiny\raisebox{-0.6pt}{#1}}}
\def\ul#1{{\setbox0\hbox{#1}\rlap{\rule[-0.9pt]{\wd0}{0.4pt}}\box0}}
\footnotesize
\setlength{\tabcolsep}{2pt}
\setlength{\aboverulesep}{0.2ex}\setlength{\belowrulesep}{0.35ex}
\begin{tabular*}{\linewidth}{@{\extracolsep{\fill}}lc cccccc cccccc@{}}
\toprule
& \textbf{Acc.} & \multicolumn{6}{c}{\textbf{Panel~A: Native $W_1\downarrow$}} & \multicolumn{6}{c}{\textbf{Panel~B: Sharp/SA$\downarrow$ (SA$\,{=}\,$1.0)}} \\
\cmidrule(lr){2-2}\cmidrule(lr){3-8}\cmidrule(l){9-14}
Dataset & RMSE$/\sigma_y$ & SA & MC & SW & MS & HSA & IV & SA & MC & SW & MS & HSA & IV \\
\midrule
Yacht & 0.050 & \textbf{0.020} & \ul{0.044} & 0.060 & 0.093 & 0.081 & 0.091 & \textbf{1.00} & \ul{1.10} & 1.53 & 1.38 & 1.68 & 1.43 \\[-3.5pt]
 & \sm{$\pm$0.021} & \sm{[.012,.053]} & \sm{[.029,.071]} & \sm{[.041,.086]} & \sm{[.084,.106]} & \sm{[.070,.096]} & \sm{[.075,.115]} &  & \sm{$\pm$0.17} & \sm{$\pm$0.46} & \sm{$\pm$0.41} & \sm{$\pm$0.62} & \sm{$\pm$0.38} \\[0pt]
Energy & 0.067 & \textbf{0.019} & 0.054 & 0.064 & \ul{0.026} & 0.031 & 0.149 & \textbf{1.00} & 1.21 & 1.25 & \ul{1.07} & 2.69 & 1.11 \\[-3.5pt]
 & \sm{$\pm$0.009} & \sm{[.012,.035]} & \sm{[.047,.062]} & \sm{[.056,.074]} & \sm{[.021,.040]} & \sm{[.022,.045]} & \sm{[.141,.160]} &  & \sm{$\pm$0.14} & \sm{$\pm$0.23} & \sm{$\pm$0.15} & \sm{$\pm$0.58} & \sm{$\pm$0.22} \\[0pt]
Naval & 0.169 & \textbf{0.007} & 0.024 & 0.064 & \ul{0.023} & 0.064 & 0.174 & \ul{1.00} & 1.20 & 3.65 & 3.08 & \textbf{0.41} & 2.04 \\[-3.5pt]
 & \sm{$\pm$0.031} & \sm{[.005,.024]} & \sm{[.017,.034]} & \sm{[.055,.077]} & \sm{[.016,.032]} & \sm{[.039,.123]} & \sm{[.168,.182]} &  & \sm{$\pm$0.12} & \sm{$\pm$0.94} & \sm{$\pm$0.40} & \sm{$\pm$0.24} & \sm{$\pm$0.34} \\[0pt]
Power & 0.237 & \textbf{0.084} & \ul{0.156} & 0.174 & 0.163 & 0.190 & 0.209 & \textbf{1.00} & \ul{1.59} & 1.86 & 1.79 & 1.90 & 1.76 \\[-3.5pt]
 & \sm{$\pm$0.015} & \sm{[.080,.089]} & \sm{[.151,.161]} & \sm{[.172,.176]} & \sm{[.161,.165]} & \sm{[.188,.192]} & \sm{[.205,.213]} &  & \sm{$\pm$0.11} & \sm{$\pm$0.08} & \sm{$\pm$0.07} & \sm{$\pm$0.08} & \sm{$\pm$0.11} \\[0pt]
Kin8nm & 0.277 & \textbf{0.060} & 0.103 & 0.096 & \ul{0.075} & 0.076 & 0.191 & \textbf{1.00} & \ul{1.27} & 1.39 & 1.30 & 1.49 & 1.47 \\[-3.5pt]
 & \sm{$\pm$0.009} & \sm{[.054,.066]} & \sm{[.100,.106]} & \sm{[.093,.100]} & \sm{[.073,.078]} & \sm{[.073,.080]} & \sm{[.188,.194]} &  & \sm{$\pm$0.05} & \sm{$\pm$0.09} & \sm{$\pm$0.05} & \sm{$\pm$0.07} & \sm{$\pm$0.08} \\[0pt]
Concrete & 0.335 & \textbf{0.042} & \ul{0.081} & 0.128 & 0.106 & 0.085 & 0.215 & \textbf{1.00} & 1.29 & 1.34 & \ul{1.24} & 1.76 & 1.37 \\[-3.5pt]
 & \sm{$\pm$0.030} & \sm{[.035,.053]} & \sm{[.075,.089]} & \sm{[.121,.136]} & \sm{[.099,.113]} & \sm{[.077,.094]} & \sm{[.212,.219]} &  & \sm{$\pm$0.07} & \sm{$\pm$0.14} & \sm{$\pm$0.11} & \sm{$\pm$0.21} & \sm{$\pm$0.16} \\[0pt]
Protein & 0.653 & 0.156 & \textbf{0.125} & 0.142 & \ul{0.132} & 0.147 & 0.223 & \textbf{1.00} & \ul{1.74} & 1.95 & 1.84 & 1.89 & 2.33 \\[-3.5pt]
 & \sm{$\pm$0.021} & \sm{[.146,.166]} & \sm{[.117,.133]} & \sm{[.139,.147]} & \sm{[.130,.134]} & \sm{[.144,.150]} & \sm{[.221,.227]} &  & \sm{$\pm$0.05} & \sm{$\pm$0.09} & \sm{$\pm$0.05} & \sm{$\pm$0.06} & \sm{$\pm$0.08} \\[0pt]
Wine & 0.905 & \ul{0.183} & 0.194 & 0.189 & \textbf{0.174} & 0.203 & 0.241 & \textbf{1.00} & \ul{2.61} & 2.94 & 2.75 & 2.71 & \ul{2.61} \\[-3.5pt]
 & \sm{$\pm$0.055} & \sm{[.175,.191]} & \sm{[.189,.200]} & \sm{[.184,.196]} & \sm{[.168,.180]} & \sm{[.198,.208]} & \sm{[.237,.245]} &  & \sm{$\pm$0.14} & \sm{$\pm$0.23} & \sm{$\pm$0.28} & \sm{$\pm$0.18} & \sm{$\pm$0.19} \\[0pt]
\bottomrule
\end{tabular*}
\end{table}

Table~\ref{tab:uci_summary} (Panel~A) reports native PIT calibration. SA attains the lowest pooled $W_1$ on six of eight datasets and, across the forty paired comparisons, is significantly better or statistically tied in 39 (Table~\ref{tab:uci_significance}). It falls behind only on Protein, one of the datasets with the largest backbone errors, where MC Dropout is significantly better. For coverage, SA is closest to the $0.90$ reference on six of eight datasets, ties MC Dropout on Protein, and lies within $0.03$ of the best result on Naval (Table~\ref{tab:uci_cov_acc}).

At the temperature-scaled operating points (Panel~B), SA produces the narrowest intervals on seven of eight datasets, and the baselines' intervals are significantly wider in 33 of 40 comparisons, by $1.2\times$ to $3.7\times$ (Table~\ref{tab:uci_significance}). The single exception is Naval, where HSA's lower point-prediction error translates into tighter scaled intervals. HSA's accuracy is competitive here because training from scratch leaves no pretrained backbone for its training-stage stochasticity to disrupt, in contrast to its lower ACC on ClimaX. A closer look at the three most accurate backbones (Yacht, Energy, Naval) reveals that SA not only has the lowest native $W_1$ but also matches the calibration of every scaled baseline, at similar coverage (Appendix~\ref{app:uci-tables}).

Similar to TimesFM and ClimaX, SA reaches its calibration level without external correction and, where coverage is comparable, stays the sharpest. These results also clarify SA's operating regime: SA is strongest where the deterministic backbone is accurate, the deployment regime scientific foundation models occupy in practice. Per-dataset PIT and sharpness plots are in Appendix~\ref{app:uci-per-dataset}.

\section{Operating-point analyses}
\label{sec:operating-point}

We use proper scores here as operating-point diagnostics on the SA family (Section~\ref{sec:lens}), evaluated on ClimaX, the highest-dimensional and most calibration-demanding setting.

Table~\ref{tab:climax_crps_decomp} (Appendix~\ref{app:crps-decomp}) reports the CRPS decomposition. At the calibration-first $\nu^\star{=}9$, SA's raw CRPS is $153.64$, second among the methods in the table. The decomposition shows that the gap is in the prediction-error term, not the distributional component: SA has by far the largest spread reward ($103.05$, against $5$--$61$ for baselines), while SWAG and MultiSWAG achieve low CRPS with spread rewards of ${\sim}5$--$6$, indicating near-deterministic ensembles. Under raw CRPS, these methods rank well by collapsing uncertainty rather than by calibrating it. This decomposition motivates a complementary question: what if $\nu$ is chosen to minimize CRPS directly, rather than the calibration objective? Sweeping $\nu$ yields a score-oriented optimum at $\nu{=}25$. At $\nu{=}25$, SA achieves CRPS$\,{=}\,135.43$ and Energy Score$\,{=}\,209.68$, the best native values among all methods, while SA's stochastic-sample accuracy rises to nearly match SWAG-BMA's. Both proper scores are jointly minimized at $\nu{=}25$ (Appendix~\ref{app:nu-ablation}).

The two operating points are interpretable consequences of their criteria. Calibration-first $\nu^\star{=}9$ minimizes the dispersion-matching loss Eq.~\eqref{eq:lsa}, giving near-uniform PIT ($W_1{=}0.040$); score-oriented $\nu{=}25$ minimizes CRPS by tightening the per-sample predictive distribution at the cost of calibration, with the PIT becoming less uniform ($W_1{=}0.093$). Both lie on the same SA family, indexed by the single concentration $\nu$. Calibration-first and score-optimal criteria do not select contradictory mechanisms; they select different points on the same controllable curve. The corresponding CRPS evaluation on UCI regression is in Appendix~\ref{app:uci-crps}, with a ranking pattern consistent with the ClimaX decomposition. Robustness of the calibration comparison to SWAG variance-scale tuning and to $\nu$-selection sensitivity is addressed in Appendices~\ref{app:swag-sweep} and~\ref{app:bo-sensitivity}.

\section{Conclusion}
\label{sec:conclusion}

We introduced stochastic attention (SA), an inference-time mechanism that replaces softmax attention with a normalized multinomial sample average and selects its single concentration parameter through a held-out calibration objective. Across time-series forecasting, gridded weather forecasting, and tabular regression, SA achieves strong native calibration without external correction, the sharpest or near-sharpest intervals at comparable coverage, and substantially lower deployment cost than retraining-heavy alternatives; the same one-parameter family further supports score-oriented operating points on the same curve. Two limitations bound this work, each addressable through a direct extension. First, $\nu$ is a single global concentration shared across heads and layers, which constrains expressivity; layer- or head-wise $\nu$ would replace it with a richer parameterization at the cost of additional calibration complexity. Second, the calibration objective is stated for continuous regression-style residuals, with classification and other structured targets reached through an appropriate residual functional. Neither extension changes the central claim. SA can also be combined with weight-space methods or epistemic neural networks~\cite{Osband2023} to capture epistemic uncertainty. The single-parameter form follows the same design pattern as established post-hoc calibration methods such as temperature scaling~\cite{Guo2017}: a deliberate choice motivated by simplicity and robustness that scale cleanly to large pretrained models. SA is therefore principled, flexible across operating points, and applicable to the regimes where calibrated uncertainty matters most for transformer-based scientific foundation models. By making calibrated uncertainty cheaper to obtain on pretrained scientific models, this work is geared toward trustworthy and reliable AI deployment in scientific applications; we are not aware of immediate negative societal risks specific to the method.

\section*{Acknowledgments}
The authors are supported by the University of Houston through the SEED program no. 000189862. T.A.~Adebiyi is supported by the National Science Foundation under Grant No.~2321040. Any opinions, findings, and conclusions or recommendations expressed in this material are those of the authors and do not necessarily reflect the views of the National Science Foundation. The authors acknowledge the use of the Carya Cluster and the advanced support from the Research Computing Data Core at the University of Houston to carry out the research presented here. The authors declare no competing interests.

\bibliographystyle{plainnat}
\bibliography{references}

\begin{thebibliography}{35}
\providecommand{\natexlab}[1]{#1}
\providecommand{\url}[1]{\texttt{#1}}
\expandafter\ifx\csname urlstyle\endcsname\relax
  \providecommand{\doi}[1]{doi: #1}\else
  \providecommand{\doi}{doi: \begingroup \urlstyle{rm}\Url}\fi

\bibitem[Balestriero et~al.(2021)Balestriero, Pesenti, and
  LeCun]{Balestriero2021}
Randall Balestriero, Jerome Pesenti, and Yann LeCun.
\newblock Learning in high dimension always amounts to extrapolation.
\newblock arXiv preprint arXiv:2110.09485, 2021.
\newblock URL \url{https://arxiv.org/abs/2110.09485}.

\bibitem[Cong et~al.(2024)Cong, Daheim, Shen, Cremers, Yokota, Khan, and
  M{\"o}llenhoff]{Cong2024}
Bai Cong, Nico Daheim, Yuesong Shen, Daniel Cremers, Rio Yokota,
  Mohammad~Emtiyaz Khan, and Thomas M{\"o}llenhoff.
\newblock Variational low-rank adaptation using {IVON}.
\newblock In \emph{NeurIPS 2024 Workshop on Fine-Tuning in Modern Machine
  Learning: Principles and Scalability}, 2024.
\newblock URL \url{https://openreview.net/forum?id=nRD5uZa2fe}.

\bibitem[Das et~al.(2024)Das, Kong, Sen, and Zhou]{Das2024}
Abhimanyu Das, Weihao Kong, Rajat Sen, and Yichen Zhou.
\newblock A decoder-only foundation model for time-series forecasting.
\newblock In \emph{Proceedings of the 41st International Conference on Machine
  Learning}, volume 235 of \emph{Proceedings of Machine Learning Research},
  pages 10148--10167. PMLR, 2024.
\newblock URL \url{https://proceedings.mlr.press/v235/das24c.html}.

\bibitem[Dawid(1984)]{Dawid1984}
A.~P. Dawid.
\newblock Present position and potential developments: Some personal views:
  Statistical theory: The prequential approach.
\newblock \emph{Journal of the Royal Statistical Society: Series A (General)},
  147\penalty0 (2):\penalty0 278--290, 1984.
\newblock \doi{10.2307/2981683}.

\bibitem[Fan et~al.(2021)Fan, Zhang, Tanwisuth, Qian, and Zhou]{Fan2021}
Xinjie Fan, Shujian Zhang, Korawat Tanwisuth, Xiaoning Qian, and Mingyuan Zhou.
\newblock Contextual dropout: An efficient sample-dependent dropout module.
\newblock In \emph{International Conference on Learning Representations}, 2021.
\newblock URL \url{https://openreview.net/forum?id=ct8_a9h1M}.

\bibitem[Ferro(2014)]{Ferro2014}
Christopher A.~T. Ferro.
\newblock Fair scores for ensemble forecasts.
\newblock \emph{Quarterly Journal of the Royal Meteorological Society},
  140\penalty0 (683):\penalty0 1917--1923, 2014.
\newblock \doi{10.1002/qj.2270}.

\bibitem[Gal and Ghahramani(2016)]{Gal2016}
Yarin Gal and Zoubin Ghahramani.
\newblock Dropout as a {B}ayesian approximation: Representing model uncertainty
  in deep learning.
\newblock In Maria~Florina Balcan and Kilian~Q. Weinberger, editors,
  \emph{Proceedings of The 33rd International Conference on Machine Learning},
  volume~48 of \emph{Proceedings of Machine Learning Research}, pages
  1050--1059, New York, New York, USA, 20--22 Jun 2016. PMLR.
\newblock URL \url{https://proceedings.mlr.press/v48/gal16.html}.

\bibitem[Gneiting and Raftery(2007)]{GneitingRaftery2007}
Tilmann Gneiting and Adrian~E. Raftery.
\newblock Strictly proper scoring rules, prediction, and estimation.
\newblock \emph{Journal of the American Statistical Association}, 102\penalty0
  (477):\penalty0 359--378, 2007.
\newblock \doi{10.1198/016214506000001437}.

\bibitem[Gneiting et~al.(2007)Gneiting, Balabdaoui, and Raftery]{Gneiting2007}
Tilmann Gneiting, Fadoua Balabdaoui, and Adrian~E. Raftery.
\newblock Probabilistic forecasts, calibration and sharpness.
\newblock \emph{Journal of the Royal Statistical Society Series B: Statistical
  Methodology}, 69\penalty0 (2):\penalty0 243--268, 2007.
\newblock \doi{10.1111/j.1467-9868.2007.00587.x}.

\bibitem[Gorishniy et~al.(2021)Gorishniy, Rubachev, Khrulkov, and
  Babenko]{Gorishniy2021}
Yury Gorishniy, Ivan Rubachev, Valentin Khrulkov, and Artem Babenko.
\newblock Revisiting deep learning models for tabular data.
\newblock In M.~Ranzato, A.~Beygelzimer, Y.~Dauphin, P.S. Liang, and J.~Wortman
  Vaughan, editors, \emph{Advances in Neural Information Processing Systems},
  volume~34, pages 18932--18943. Curran Associates, Inc., 2021.
\newblock URL
  \url{https://proceedings.neurips.cc/paper_files/paper/2021/file/9d86d83f925f2149e9edb0ac3b49229c-Paper.pdf}.

\bibitem[Gruber and Buettner(2022)]{Gruber2022}
Sebastian Gruber and Florian Buettner.
\newblock Better uncertainty calibration via proper scores for classification
  and beyond.
\newblock In S.~Koyejo, S.~Mohamed, A.~Agarwal, D.~Belgrave, K.~Cho, and A.~Oh,
  editors, \emph{Advances in Neural Information Processing Systems}, volume~35,
  pages 8618--8632. Curran Associates, Inc., 2022.

\bibitem[Guo et~al.(2017)Guo, Pleiss, Sun, and Weinberger]{Guo2017}
Chuan Guo, Geoff Pleiss, Yu~Sun, and Kilian~Q. Weinberger.
\newblock On calibration of modern neural networks.
\newblock In Doina Precup and Yee~Whye Teh, editors, \emph{Proceedings of the
  34th International Conference on Machine Learning}, volume~70 of
  \emph{Proceedings of Machine Learning Research}, pages 1321--1330. PMLR,
  06--11 Aug 2017.
\newblock URL \url{https://proceedings.mlr.press/v70/guo17a.html}.

\bibitem[Heo et~al.(2018)Heo, Lee, Kim, Lee, Kim, Yang, and Hwang]{Heo2018}
Jay Heo, Hae~Beom Lee, Saehoon Kim, Juho Lee, Kwang~Joon Kim, Eunho Yang, and
  Sung~Ju Hwang.
\newblock Uncertainty-aware attention for reliable interpretation and
  prediction.
\newblock In S.~Bengio, H.~Wallach, H.~Larochelle, K.~Grauman, N.~Cesa-Bianchi,
  and R.~Garnett, editors, \emph{Advances in Neural Information Processing
  Systems}, volume~31. Curran Associates, Inc., 2018.
\newblock URL
  \url{https://proceedings.neurips.cc/paper_files/paper/2018/file/285e19f20beded7d215102b49d5c09a0-Paper.pdf}.

\bibitem[Holm(1979)]{Holm1979}
Sture Holm.
\newblock A simple sequentially rejective multiple test procedure.
\newblock \emph{Scandinavian Journal of Statistics}, 6\penalty0 (2):\penalty0
  65--70, 1979.
\newblock URL \url{https://www.jstor.org/stable/4615733}.

\bibitem[Hu et~al.(2022)Hu, Shen, Wallis, Allen-Zhu, Li, Wang, Wang, and
  Chen]{Hu2022}
Edward~J Hu, Yelong Shen, Phillip Wallis, Zeyuan Allen-Zhu, Yuanzhi Li, Shean
  Wang, Lu~Wang, and Weizhu Chen.
\newblock Lo{RA}: Low-rank adaptation of large language models.
\newblock In \emph{International Conference on Learning Representations}, 2022.

\bibitem[Izmailov et~al.(2018)Izmailov, Podoprikhin, Garipov, Vetrov, and
  Wilson]{Izmailov2018}
Pavel Izmailov, Dmitrii Podoprikhin, Timur Garipov, Dmitry~P. Vetrov, and
  Andrew~Gordon Wilson.
\newblock Averaging weights leads to wider optima and better generalization.
\newblock In Amir Globerson and Ricardo Silva, editors, \emph{Proceedings of
  the Thirty-Fourth Conference on Uncertainty in Artificial Intelligence},
  pages 876--885, Monterey, California, USA, August 2018. AUAI Press.

\bibitem[Kuleshov and Deshpande(2022)]{Kuleshov2022}
Volodymyr Kuleshov and Shachi Deshpande.
\newblock Calibrated and sharp uncertainties in deep learning via density
  estimation.
\newblock In Kamalika Chaudhuri, Stefanie Jegelka, Le~Song, Csaba Szepesvari,
  Gang Niu, and Sivan Sabato, editors, \emph{Proceedings of the 39th
  International Conference on Machine Learning}, volume 162 of
  \emph{Proceedings of Machine Learning Research}, pages 11683--11693. PMLR,
  17--23 Jul 2022.
\newblock URL \url{https://proceedings.mlr.press/v162/kuleshov22a.html}.

\bibitem[Kuleshov et~al.(2018)Kuleshov, Fenner, and Ermon]{Kuleshov2018}
Volodymyr Kuleshov, Nathan Fenner, and Stefano Ermon.
\newblock Accurate uncertainties for deep learning using calibrated regression.
\newblock In Jennifer Dy and Andreas Krause, editors, \emph{Proceedings of the
  35th International Conference on Machine Learning}, volume~80 of
  \emph{Proceedings of Machine Learning Research}, pages 2796--2804. PMLR,
  10--15 Jul 2018.
\newblock URL \url{https://proceedings.mlr.press/v80/kuleshov18a.html}.

\bibitem[Lakshminarayanan et~al.(2017)Lakshminarayanan, Pritzel, and
  Blundell]{Lakshminarayanan2017}
Balaji Lakshminarayanan, Alexander Pritzel, and Charles Blundell.
\newblock Simple and scalable predictive uncertainty estimation using deep
  ensembles.
\newblock In I.~Guyon, U.~Von Luxburg, S.~Bengio, H.~Wallach, R.~Fergus,
  S.~Vishwanathan, and R.~Garnett, editors, \emph{Advances in Neural
  Information Processing Systems}, volume~30. Curran Associates, Inc., 2017.
\newblock URL
  \url{https://proceedings.neurips.cc/paper_files/paper/2017/file/9ef2ed4b7fd2c810847ffa5fa85bce38-Paper.pdf}.

\bibitem[Lei et~al.(2018)Lei, G'Sell, Rinaldo, Tibshirani, and
  Wasserman]{Lei2018}
Jing Lei, Max G'Sell, Alessandro Rinaldo, Ryan~J. Tibshirani, and Larry
  Wasserman.
\newblock Distribution-free predictive inference for regression.
\newblock \emph{Journal of the American Statistical Association}, 113\penalty0
  (523):\penalty0 1094--1111, 2018.
\newblock \doi{10.1080/01621459.2017.1307116}.

\bibitem[Maddox et~al.(2019)Maddox, Izmailov, Garipov, Vetrov, and
  Wilson]{Maddox2019}
Wesley~J Maddox, Pavel Izmailov, Timur Garipov, Dmitry~P Vetrov, and
  Andrew~Gordon Wilson.
\newblock A simple baseline for {B}ayesian uncertainty in deep learning.
\newblock In \emph{Advances in Neural Information Processing Systems},
  volume~32, pages 13153--13164. Curran Associates, Inc., 2019.

\bibitem[Nadeau and Bengio(2003)]{Nadeau2003}
Claude Nadeau and Yoshua Bengio.
\newblock Inference for the generalization error.
\newblock \emph{Machine Learning}, 52\penalty0 (3):\penalty0 239--281, 2003.
\newblock \doi{10.1023/A:1024068626366}.

\bibitem[Nguyen et~al.(2023)Nguyen, Brandstetter, Kapoor, Gupta, and
  Grover]{NguyenClimaX23}
Tung Nguyen, Johannes Brandstetter, Ashish Kapoor, Jayesh~K. Gupta, and Aditya
  Grover.
\newblock {ClimaX}: A foundation model for weather and climate.
\newblock In \emph{Proceedings of the 40th International Conference on Machine
  Learning}, volume 202 of \emph{Proceedings of Machine Learning Research},
  pages 25904--25938. PMLR, 2023.
\newblock URL \url{https://proceedings.mlr.press/v202/nguyen23a.html}.

\bibitem[Onal et~al.(2024)Onal, Fl{\"o}ge, Caldwell, Sheverdin, and
  Fortuin]{Onal2024}
Emre Onal, Klemens Fl{\"o}ge, Emma Caldwell, Arsen Sheverdin, and Vincent
  Fortuin.
\newblock Gaussian stochastic weight averaging for {B}ayesian low-rank
  adaptation of large language models.
\newblock In \emph{Sixth Symposium on Advances in Approximate Bayesian
  Inference - Non Archival Track}, 2024.
\newblock URL \url{https://openreview.net/forum?id=LZrCBQBCzl}.

\bibitem[Osband et~al.(2023)Osband, Wen, Asghari, Dwaracherla, Ibrahimi, Lu,
  and Van~Roy]{Osband2023}
Ian Osband, Zheng Wen, Seyed~Mohammad Asghari, Vikranth Dwaracherla, Morteza
  Ibrahimi, Xiuyuan Lu, and Benjamin Van~Roy.
\newblock Epistemic neural networks.
\newblock In A.~Oh, T.~Naumann, A.~Globerson, K.~Saenko, M.~Hardt, and
  S.~Levine, editors, \emph{Advances in Neural Information Processing Systems},
  volume~36, pages 2795--2823. Curran Associates, Inc., 2023.
\newblock \doi{10.52202/075280-0124}.
\newblock URL \url{https://www.proceedings.com/075280-0124.html}.

\bibitem[Papadopoulos et~al.(2002)Papadopoulos, Proedrou, Vovk, and
  Gammerman]{Papadopoulos2002}
Harris Papadopoulos, Kostas Proedrou, Volodya Vovk, and Alex Gammerman.
\newblock Inductive confidence machines for regression.
\newblock In Tapio Elomaa, Heikki Mannila, and Hannu Toivonen, editors,
  \emph{Machine Learning: {ECML} 2002}, volume 2430 of \emph{Lecture Notes in
  Computer Science}, pages 345--356, Berlin, Heidelberg, 2002. Springer.
\newblock \doi{10.1007/3-540-36755-1_29}.

\bibitem[Pei et~al.(2022)Pei, Wang, and Szarvas]{Pei2022}
Jiahuan Pei, Cheng Wang, and György Szarvas.
\newblock Transformer uncertainty estimation with hierarchical stochastic
  attention.
\newblock \emph{Proceedings of the AAAI Conference on Artificial Intelligence},
  36\penalty0 (10):\penalty0 11147--11155, Jun. 2022.
\newblock \doi{10.1609/aaai.v36i10.21364}.
\newblock URL \url{https://ojs.aaai.org/index.php/AAAI/article/view/21364}.

\bibitem[Shen et~al.(2024)Shen, Daheim, Cong, Nickl, Marconi, Raoul, Yokota,
  Gurevych, Cremers, Khan, and M\"{o}llenhoff]{Shen2024}
Yuesong Shen, Nico Daheim, Bai Cong, Peter Nickl, Gian~Maria Marconi, Bazan
  Clement Emile~Marcel Raoul, Rio Yokota, Iryna Gurevych, Daniel Cremers,
  Mohammad~Emtiyaz Khan, and Thomas M\"{o}llenhoff.
\newblock Variational learning is effective for large deep networks.
\newblock In \emph{Proceedings of the 41st International Conference on Machine
  Learning}, volume 235 of \emph{Proceedings of Machine Learning Research},
  pages 44665--44686. PMLR, 21--27 Jul 2024.

\bibitem[Vovk et~al.(2005)Vovk, Gammerman, and Shafer]{Vovk2005}
Vladimir Vovk, Alexander Gammerman, and Glenn Shafer.
\newblock \emph{Algorithmic Learning in a Random World}.
\newblock Springer, New York, 2005.
\newblock \doi{10.1007/b106715}.

\bibitem[Wilcoxon(1945)]{Wilcoxon1945}
Frank Wilcoxon.
\newblock Individual comparisons by ranking methods.
\newblock \emph{Biometrics Bulletin}, 1\penalty0 (6):\penalty0 80--83, 1945.

\bibitem[Wilks(1941)]{Wilks1941}
S.~S. Wilks.
\newblock Determination of sample sizes for setting tolerance limits.
\newblock \emph{The Annals of Mathematical Statistics}, 12\penalty0
  (1):\penalty0 91--96, 1941.
\newblock \doi{10.1214/aoms/1177731788}.

\bibitem[Wilson and Izmailov(2020)]{Wilson2020}
Andrew~Gordon Wilson and Pavel Izmailov.
\newblock Bayesian deep learning and a probabilistic perspective of
  generalization.
\newblock In \emph{Advances in Neural Information Processing Systems},
  volume~33, pages 4697--4708. Curran Associates, Inc., 2020.

\bibitem[Yadav and Zhang(2026)]{Yadav2025sobo}
Akash Yadav and Ruda Zhang.
\newblock Bayesian optimization under uncertainty for training a scale
  parameter in stochastic models.
\newblock \emph{ASCE-ASME Journal of Risk and Uncertainty in Engineering
  Systems, Part A: Civil Engineering}, 12\penalty0 (4):\penalty0 04026050,
  2026.
\newblock \doi{10.1061/AJRUA6.RUENG-1854}.

\bibitem[Ye et~al.(2024)Ye, Yang, Pang, Wang, Wong, Yilmaz, Shi, and
  Tu]{Ye2024}
Fanghua Ye, Mingming Yang, Jianhui Pang, Longyue Wang, Derek~F. Wong, Emine
  Yilmaz, Shuming Shi, and Zhaopeng Tu.
\newblock Benchmarking {LLMs} via uncertainty quantification.
\newblock In A.~Globerson, L.~Mackey, D.~Belgrave, A.~Fan, U.~Paquet,
  J.~Tomczak, and C.~Zhang, editors, \emph{Advances in Neural Information
  Processing Systems}, volume~37, pages 15356--15385. Curran Associates, Inc.,
  2024.
\newblock \doi{10.52202/079017-0491}.

\bibitem[Zhou et~al.(2021)Zhou, Zhang, Peng, Zhang, Li, Xiong, and
  Zhang]{Zhou2021}
Haoyi Zhou, Shanghang Zhang, Jieqi Peng, Shuai Zhang, Jianxin Li, Hui Xiong,
  and Wancai Zhang.
\newblock Informer: Beyond efficient transformer for long sequence time-series
  forecasting.
\newblock \emph{Proceedings of the AAAI Conference on Artificial Intelligence},
  35\penalty0 (12):\penalty0 11106--11115, May 2021.
\newblock \doi{10.1609/aaai.v35i12.17325}.
\newblock URL \url{https://ojs.aaai.org/index.php/AAAI/article/view/17325}.

\end{thebibliography}

\clearpage
\appendix
\renewcommand{\topfraction}{0.9}
\renewcommand{\textfraction}{0.07}
\renewcommand{\floatpagefraction}{0.8}
\setcounter{topnumber}{3}
\setcounter{totalnumber}{4}
\makeatletter
\setlength{\@fptop}{0pt}
\setlength{\@fpsep}{14pt plus 2fil}
\setlength{\@fpbot}{0pt plus 1fil}
\makeatother

\section{Proofs}
\label{app:proofs}

For completeness, we collect here the proofs of the main stochastic-attention properties stated in Propositions~\ref{prop:unbiased}--\ref{prop:variance}. Throughout, we fix an attention row index $t$ and condition on the deterministic attention weights $\boldsymbol{\pi}_t\in\Delta^{n_k-1}$ and the value matrix $\mathbf{V}\in\mathbb{R}^{n_k\times d_v}$.
We regard $\boldsymbol{\pi}_t$, $\mathbf{w}_t$, and $\widetilde{\boldsymbol{\pi}}_t$ as column vectors. The attention output in the main text is written as the row vector
$\widetilde{\mathbf{o}}_t=\widetilde{\boldsymbol{\pi}}_t^\intercal\mathbf{V}$; therefore, when writing covariance matrices for attention outputs, we take the covariance of the corresponding column vector $\widetilde{\mathbf{o}}_t^\intercal$.
Recall that
\[
\mathbf{w}_t \sim \operatorname{Multinomial}(\nu,\boldsymbol{\pi}_t),
\qquad
\widetilde{\boldsymbol{\pi}}_t = \frac{1}{\nu}\mathbf{w}_t,
\qquad
\widetilde{\mathbf{o}}_t = \widetilde{\boldsymbol{\pi}}_t^\intercal \mathbf{V},
\qquad
\mathbf{o}_t = \boldsymbol{\pi}_t^\intercal \mathbf{V}.
\]

\begin{proof}[Proof of Proposition~\ref{prop:unbiased}]
We prove the three claims in turn.

\paragraph{Conditional mean of the stochastic weights.}
For a multinomial random vector,
\[
\mathbb{E}\!\left[(\mathbf{w}_t)_j \mid \boldsymbol{\pi}_t\right]
=
\nu (\boldsymbol{\pi}_t)_j,
\qquad j=1,\dots,n_k.
\]
Therefore,
\[
\mathbb{E}\!\left[(\widetilde{\boldsymbol{\pi}}_t)_j \mid \boldsymbol{\pi}_t\right]
=
\frac{1}{\nu}\,
\mathbb{E}\!\left[(\mathbf{w}_t)_j \mid \boldsymbol{\pi}_t\right]
=
(\boldsymbol{\pi}_t)_j.
\]
Since this holds componentwise,
\[
\mathbb{E}\!\left[\widetilde{\boldsymbol{\pi}}_t \mid \boldsymbol{\pi}_t\right]
=
\boldsymbol{\pi}_t.
\]

\paragraph{Conditional mean of the stochastic attention output.}
Using linearity of conditional expectation and the fact that $\mathbf{V}$ is fixed under the conditioning,
\[
\mathbb{E}\!\left[\widetilde{\mathbf{o}}_t^\intercal \mid \boldsymbol{\pi}_t,\mathbf{V}\right]
=
\mathbf{V}^\intercal
\mathbb{E}\!\left[\widetilde{\boldsymbol{\pi}}_t \mid \boldsymbol{\pi}_t\right]
=
\mathbf{V}^\intercal \boldsymbol{\pi}_t
=
\mathbf{o}_t^\intercal.
\]
Equivalently,
\[
\mathbb{E}\!\left[\widetilde{\mathbf{o}}_t \mid \boldsymbol{\pi}_t,\mathbf{V}\right]
=
\mathbf{o}_t.
\]

\paragraph{Deterministic recovery as $\nu\to\infty$.}
Let $z_t^{(m)} \iid \operatorname{Categorical}(\boldsymbol{\pi}_t)$ for $m=1,2,\ldots$, and for each integer $\nu\geq 1$ define
\[
\widetilde{\boldsymbol{\pi}}_{t,\nu}
=
\frac{1}{\nu}\sum_{m=1}^{\nu}\mathbf{e}_{z_t^{(m)}}.
\]
This has the same distribution as $\nu^{-1}\mathbf{w}_t$ when $\mathbf{w}_t\sim\operatorname{Multinomial}(\nu,\boldsymbol{\pi}_t)$. For each coordinate $j$,
\[
(\widetilde{\boldsymbol{\pi}}_{t,\nu})_j
=
\frac{1}{\nu}\sum_{m=1}^{\nu}\mathds{1}\{z_t^{(m)}=j\}.
\]
By the strong law of large numbers,
\[
(\widetilde{\boldsymbol{\pi}}_{t,\nu})_j
\xrightarrow[\nu\to\infty]{\mathrm{a.s.}}
(\boldsymbol{\pi}_t)_j,
\qquad j=1,\dots,n_k.
\]
Since $n_k$ is finite, this implies
\[
\widetilde{\boldsymbol{\pi}}_{t,\nu}
\xrightarrow[\nu\to\infty]{\mathrm{a.s.}}
\boldsymbol{\pi}_t.
\]
The map $\mathbf{u}\mapsto \mathbf{u}^\intercal\mathbf{V}$ is continuous, so
\[
\widetilde{\mathbf{o}}_{t,\nu}
=
\widetilde{\boldsymbol{\pi}}_{t,\nu}^\intercal\mathbf{V}
\xrightarrow[\nu\to\infty]{\mathrm{a.s.}}
\boldsymbol{\pi}_t^\intercal\mathbf{V}
=
\mathbf{o}_t.
\]
\end{proof}

\begin{proof}[Proof of Proposition~\ref{prop:variance}]
We again condition on $(\boldsymbol{\pi}_t,\mathbf{V})$.

\paragraph{Covariance of the stochastic weights.}
Write $\boldsymbol{\pi}_t
=
\bigl((\boldsymbol{\pi}_t)_1,\dots,(\boldsymbol{\pi}_t)_{n_k}\bigr)^\intercal$.
Standard multinomial moment identities give $\Cov\!\left(\mathbf{w}_t \mid \boldsymbol{\pi}_t\right)
=
\nu\Big(
\diag(\boldsymbol{\pi}_t)
-
\boldsymbol{\pi}_t\boldsymbol{\pi}_t^\intercal
\Big)$.
Because $\widetilde{\boldsymbol{\pi}}_t=\nu^{-1}\mathbf{w}_t$, covariance scaling yields
\[
\Cov\!\left(\widetilde{\boldsymbol{\pi}}_t \mid \boldsymbol{\pi}_t\right)
=
\frac{1}{\nu}
\Big(
\diag(\boldsymbol{\pi}_t)
-
\boldsymbol{\pi}_t\boldsymbol{\pi}_t^\intercal
\Big).
\]

\paragraph{Propagation to the stochastic attention output.}
Since the column representation of the stochastic attention output is $\widetilde{\mathbf{o}}_t^\intercal
=
\mathbf{V}^\intercal\widetilde{\boldsymbol{\pi}}_t$,
we apply $\Cov(\mathbf{A}\mathbf{x})=\mathbf{A}\Cov(\mathbf{x})\mathbf{A}^\intercal$ with $\mathbf{A}=\mathbf{V}^\intercal$. This gives
$\Cov\!\left(\widetilde{\mathbf{o}}_t^\intercal \mid \boldsymbol{\pi}_t,\mathbf{V}\right)
=
\mathbf{V}^\intercal
\Cov\!\left(\widetilde{\boldsymbol{\pi}}_t \mid \boldsymbol{\pi}_t\right)
\mathbf{V}$.
Substituting the covariance of $\widetilde{\boldsymbol{\pi}}_t$ yields
\[
\Cov\!\left(\widetilde{\mathbf{o}}_t^\intercal \mid \boldsymbol{\pi}_t,\mathbf{V}\right)
=
\frac{1}{\nu}\,
\mathbf{V}^\intercal
\Big(
\diag(\boldsymbol{\pi}_t)
-
\boldsymbol{\pi}_t\boldsymbol{\pi}_t^\intercal
\Big)
\mathbf{V}.
\]
\end{proof}

\section{Calibrating $\nu$}
\label{app:calibration}

This appendix gives the routine that selects $\nu$ (Appendix~\ref{app:bo-nu}), its sensitivity on ClimaX (Appendix~\ref{app:bo-sensitivity}), and a test of whether separate concentrations per layer or per head improve on a single one (Appendix~\ref{app:nu-scope}).

\subsection{Bayesian Optimization Routine}
\label{app:bo-nu}

This subsection gives the practical optimization procedure used to calibrate the
stochastic-attention concentration parameter $\nu$. In the main text, $\nu$ is selected by
minimizing the calibration loss $\mathcal{L}_{\mathrm{SA}}(\nu)$ over an admissible integer
search domain $\Xi\subset\mathbb{N}$.
Here, we describe how that optimization is carried out in
practice when only noisy Monte Carlo estimates of the loss and the stochastic-deviation scale are
available, following \citet{Yadav2025sobo}, which derives and analyzes the routine.

\paragraph{What is being optimized.}
The target quantity is the calibration loss from Eq.~\eqref{eq:lsa},
\[
	\mathcal{L}_{\mathrm{SA}}(\nu)
	=
	\mathbb{E}_{(\mathbf{x},\mathbf{y})\sim\mathcal{D}_{\mathrm{cal}}}
	\left[
		\mathbb{E}\!\left[
			\bigl(R_\nu(\mathbf{x})-Z(\mathbf{x},\mathbf{y})\bigr)^2
			\,\middle|\, \mathbf{x},\mathbf{y}
		\right]
	\right],
\]
where
\[
	R_\nu(\mathbf{x})
	:=
	\left\|
	f_{\boldsymbol{\theta},\nu}(\mathbf{x})
	-
	f_{\boldsymbol{\theta}}(\mathbf{x})
	\right\|_2,
	\qquad
	Z(\mathbf{x},\mathbf{y})
	:=
	\left\|
	\mathbf{y}-f_{\boldsymbol{\theta}}(\mathbf{x})
	\right\|_2.
\]
The conditioning on $(\mathbf{x},\mathbf{y})$ emphasizes that the inner expectation is only
over stochastic attention samples; the realized residual magnitude is fixed once the calibration example is fixed.
As in Sections~\ref{sec:method}--\ref{sec:calibration}, the model parameters
$\boldsymbol{\theta}$ are frozen throughout. The parameter $\nu$ controls the
stochastic-attention law, whereas repeated stochastic forward passes are used only to estimate
the expectations appearing in the loss.

\paragraph{Why the evaluations are noisy.}
For a fixed candidate $\nu\in\Xi$, we do not observe $\mathcal{L}_{\mathrm{SA}}(\nu)$ exactly.
One evaluation draws a mini-batch $\mathcal{B}\subset\mathcal{D}_{\mathrm{cal}}$, runs one stochastic forward pass per
example, and returns a Monte Carlo estimate $\widehat{\ell}(\nu)$ of $\mathcal{L}_{\mathrm{SA}}(\nu)$ together with an
empirical stochastic-deviation scale $\widehat{s}(\nu)$. Their randomness comes from two sources: (i) the mini-batch,
and (ii) stochastic attention sampling within each forward pass.
The Bayesian optimization surrogate is fit to the scale, not to the calibration loss itself, because the scale follows
an approximate power law in $\nu$ whereas the loss is U-shaped; the loss of every evaluation is recorded as a check on
the selected value (Appendix~\ref{app:bo-sensitivity}).

\paragraph{Overall calibration routine.}
Algorithm~\ref{alg:calibrate_nu_appendix} describes the full post-hoc calibration loop. A target scale
$\widehat{s}_0$ is computed once. An initial design of $K_0$ values, evenly spaced in $\ln\nu$ over $\Xi$, is evaluated
first. Each of the $K$ iterations then fits the surrogate to the scale history, proposes $q$ candidates by Thompson
sampling (Algorithm~\ref{alg:bo_nu}), and evaluates them on a fresh mini-batch. The returned $\nu^\star$ is the point of
$\Xi$ nearest to the minimizer of the surrogate objective at the fitted parameters; $\Pi_{\Xi}$ denotes this
nearest-point projection onto $\Xi$. On ClimaX, $\Xi=\{1,\ldots,50\}$, $K_0=40$, $q=10$ and $K=50$, i.e., 540
evaluations.

\begin{algorithm}[!t]
  \caption{Calibration of the stochastic-attention concentration parameter $\nu$}
  \label{alg:calibrate_nu_appendix}
  \small
  \begin{algorithmic}[1]
    \Require Trained parameters $\boldsymbol{\theta}$; calibration set $\mathcal{D}_{\mathrm{cal}}$;
    batch sampler $\textsc{SampleBatch}(\mathcal{D}_{\mathrm{cal}})$;
    integer search domain $\Xi=\{\nu_{\min},\ldots,\nu_{\max}\}\subset\mathbb{N}$;
    initial design size $K_0$; candidates per iteration $q$; number of iterations $K$.
    \Ensure Calibrated concentration parameter $\nu^\star$.

    \Function{EvalStats}{$\nu,\mathcal{B}$}
      \ForAll{$(\mathbf{x}_i,\mathbf{y}_i)\in\mathcal{B}$}
        \State $\widehat{\mathbf{y}}_i\gets f_{\boldsymbol{\theta}}(\mathbf{x}_i)$; \ $\widetilde{\mathbf{y}}_i\gets f_{\boldsymbol{\theta},\nu}(\mathbf{x}_i)$ \Comment{one stochastic pass}
        \State $r_i\gets\|\mathbf{y}_i-\widehat{\mathbf{y}}_i\|_2$; \ $\delta_i\gets\|\widetilde{\mathbf{y}}_i-\widehat{\mathbf{y}}_i\|_2$
      \EndFor
      \State \Return $\widehat{\ell}\gets\frac{1}{|\mathcal{B}|}\sum_{i}(\delta_i-r_i)^2$, \ $\widehat{s}\gets\frac{1}{|\mathcal{B}|}\sum_{i}\delta_i$
    \EndFunction

    \State Freeze $\boldsymbol{\theta}$ and disable other stochasticity.
    \State $\mathcal{B}_0\gets\textsc{SampleBatch}(\mathcal{D}_{\mathrm{cal}})$; compute $\widehat{s}_0$ on $\mathcal{B}_0$ by Eq.~\eqref{eq:bo_target_scale}
    \State $\mathcal{H}\gets\emptyset$
    \For{$k=1$ \textbf{to} $K_0$} \Comment{initial design, evenly spaced in $\ln\nu$}
      \State $\nu\gets\Pi_{\Xi}\bigl(\exp\bigl(\ln\nu_{\min}+\tfrac{k-1}{K_0-1}\ln(\nu_{\max}/\nu_{\min})\bigr)\bigr)$
      \State $(\widehat{\ell},\widehat{s})\gets\Call{EvalStats}{\nu,\mathcal{B}_0}$; \ $\mathcal{H}\gets\mathcal{H}\cup\{(\nu,\widehat{s},\widehat{\ell})\}$
    \EndFor
    \For{$k=1$ \textbf{to} $K$}
      \State $\mathcal{H}_s\gets\{(\nu,\widehat{s}):(\nu,\widehat{s},\widehat{\ell})\in\mathcal{H}\}$
      \State $(\nu_{k,1},\ldots,\nu_{k,q})\gets\textsc{BayesOptSuggest}(\mathcal{H}_s,\Xi,\widehat{s}_0,q)$
      \State $\mathcal{B}_k\gets\textsc{SampleBatch}(\mathcal{D}_{\mathrm{cal}})$
      \For{$j=1$ \textbf{to} $q$}
        \State $(\widehat{\ell},\widehat{s})\gets\Call{EvalStats}{\nu_{k,j},\mathcal{B}_k}$; \ $\mathcal{H}\gets\mathcal{H}\cup\{(\nu_{k,j},\widehat{s},\widehat{\ell})\}$
      \EndFor
    \EndFor
    \State Fit Eq.~\eqref{eq:bo_loglog} by least squares to all $(\nu,\widehat{s})$ in $\mathcal{H}$, giving $(\widehat{a},\widehat{b},\widehat{\varepsilon}^{2})$
    \State $\nu^\star\gets\Pi_{\Xi}\bigl((\widehat{s}_0/(\widehat{b}\exp(\tfrac{3}{2}\widehat{\varepsilon}^{2})))^{1/\widehat{a}}\bigr)$ \Comment{Eq.~\eqref{eq:bo_continuous_minimizer}}
    \State \Return $\nu^\star$
  \end{algorithmic}
\end{algorithm}

\paragraph{BO-under-uncertainty surrogate.}
To exploit the fact that this is a one-dimensional tuning problem, we use the
BO-under-uncertainty framework of \cite{Yadav2025sobo}. Each noisy evaluation is summarized
through a positive stochastic-deviation scale and compared against a fixed positive target
scale. In our setting, the target scale is computed once, on the initial calibration sample $\mathcal{B}_0$, as
\begin{equation}
	\label{eq:bo_target_scale}
	\widehat{s}_0
	:=
	\frac{1}{|\mathcal{B}_0|}
	\sum_{(\mathbf{x}_i,\mathbf{y}_i)\in\mathcal{B}_0}
	\left\|
	\mathbf{y}_i-f_{\boldsymbol{\theta}}(\mathbf{x}_i)
	\right\|_2,
\end{equation}
and the stochastic-deviation scale returned by an evaluation at candidate $\nu$ on a mini-batch $\mathcal{B}$ is
\begin{equation}
	\label{eq:bo_scale_proxy}
	\widehat{s}(\nu)
	:=
	\frac{1}{|\mathcal{B}|}
	\sum_{(\mathbf{x}_i,\mathbf{y}_i)\in\mathcal{B}}
	\left\|
	f_{\boldsymbol{\theta},\nu}(\mathbf{x}_i)
	-
	f_{\boldsymbol{\theta}}(\mathbf{x}_i)
	\right\|_2,
\end{equation}
where each $f_{\boldsymbol{\theta},\nu}(\mathbf{x}_i)$ is one stochastic forward pass. The quantity $\widehat{s}(\nu)$
measures the average per-example stochastic deviation induced by stochastic attention at concentration level $\nu$,
while $\widehat{s}_0$ measures the corresponding average per-example residual scale.

We model the dependence of the stochastic-deviation scale on $\nu$ through the Bayesian
log--log surrogate
\begin{equation}
	\label{eq:bo_loglog}
	\ln \widehat{s}(\nu)
	=
	a\ln \nu + \ln b + \varepsilon z,
	\qquad
	z\sim\mathcal{N}(0,1),
\end{equation}
with the noninformative prior $p(a,\ln b,\varepsilon^2)\propto\varepsilon^{-2}$, which implies
\[
	\widehat{s}(\nu) = b\,\nu^a\zeta,
	\qquad
	\zeta\sim\log\mathcal{N}(0,\varepsilon^2).
\]
This surrogate is fit to the scale history $\mathcal{H}_s$ of evaluated pairs $(\nu,\widehat{s})$, not to
the U-shaped loss history. This distinction is important because the stochastic deviation
scale is approximately monotone in $\nu$, whereas the calibration loss is minimized when that
scale matches the residual scale.

Conditional on sampled surrogate parameters $(a,\ln b,\varepsilon^2)$, the next candidate can be
selected by minimizing the expected squared discrepancy between the surrogate scale and the
target scale:
\begin{equation}
	\label{eq:bo_surrogate_loss}
	\mathbb{E}\!\left[
		\bigl( b\,\nu^a\zeta - \widehat{s}_0\bigr)^2
		\,\middle|\,
		a,b,\varepsilon^2,\widehat{s}_0
		\right]
	=
	\Var\!\left(b\nu^a\zeta\right)
	+
	\left(
	\mathbb{E}[b\nu^a\zeta] - \widehat{s}_0
	\right)^2.
\end{equation}
For $\zeta\sim\log\mathcal{N}(0,\varepsilon^2)$ and $a\neq 0$, the positive continuous
minimizer is
\begin{equation}
	\label{eq:bo_continuous_minimizer}
	\nu_{\mathrm{cont}}
	=
    \left(\frac{\widehat{s}_0}{b \exp(\tfrac{3}{2} \varepsilon^2)}\right)^{1/a}.
\end{equation}
The proposal is then projected back to the admissible integer domain $\Xi$.

\paragraph{BO suggestion step.}
Algorithm~\ref{alg:bo_nu} gives the internal proposal mechanism used at each iteration of the
outer calibration loop. The routine fits the log--log surrogate to the scale history, draws
$q$ sets of surrogate parameters from the posterior (Thompson sampling), computes the continuous
minimizer for each draw, and projects these values back to the admissible integer domain $\Xi$.

\begin{algorithm}[!t]
  \caption{Scale-based BO suggestion for $\nu$}
  \label{alg:bo_nu}
  \small
  \begin{algorithmic}[1]
    \Require Scale history $\mathcal{H}_s$ of evaluated pairs $(\nu,\widehat{s})$; integer search domain
    $\Xi\subset\mathbb{N}$; target scale $\widehat{s}_0$; number of candidates $q$.
    \Ensure Next candidates $\nu_1,\ldots,\nu_q$.

    \State Fit the posterior $p(a,\ln b,\varepsilon^2\mid\mathcal{H}_s)$ under
    $\ln\widehat{s} = a\ln\nu + \ln b + \varepsilon z$ for $(\nu,\widehat{s})\in\mathcal{H}_s$.
    \For{$j=1$ \textbf{to} $q$}
      \State Draw $(\widetilde{a},\widetilde{\ln b},\widetilde{\varepsilon}^2)\sim p(a,\ln b,\varepsilon^2\mid\mathcal{H}_s)$ and set $\widetilde{b}=\exp(\widetilde{\ln b})$.
      \State $\nu_{\mathrm{cont}}\gets\bigl(\widehat{s}_0/(\widetilde{b}\exp(\tfrac{3}{2}\widetilde{\varepsilon}^2))\bigr)^{1/\widetilde{a}}$ \Comment{Eq.~\eqref{eq:bo_continuous_minimizer}}
      \State $\nu_j\gets\Pi_{\Xi}(\nu_{\mathrm{cont}})$
    \EndFor
    \State \Return $\nu_1,\ldots,\nu_q$
  \end{algorithmic}
\end{algorithm}

\subsection{BO Sensitivity and Robustness}
\label{app:bo-sensitivity}

On ClimaX (calibration year 2016), BO made 540 evaluations and visited 26 values of $\nu$ between 1 and 50. The per-evaluation calibration loss $\widehat{\ell}(\nu)$ falls by three orders of magnitude from $\nu=1$ to a minimum at $\nu=9$--$10$ and rises on both sides, so BO's own log is a coarse sweep of the search range that a dense grid search would only repeat, and BO spent most of its budget (445 evaluations) at $\nu=9$. Figure~\ref{fig:bo_sensitivity}a shows the surrogate fitted to these evaluations: its objective has a single, well-identified minimum at $8.93$, which rounds to $\nu^\star{=}9$, and the minimizers of 300 posterior samples lie between $8.83$ and $9.05$ (Figure~\ref{fig:bo_sensitivity}b), so all of them round to 9. Across the 20 runs of the statistical study, each with its own backbone, calibration year, and seed, BO selects $\nu^\star=10.55\pm1.39$ (range 7--12).

\begin{figure}[tp]
\centering
\includegraphics{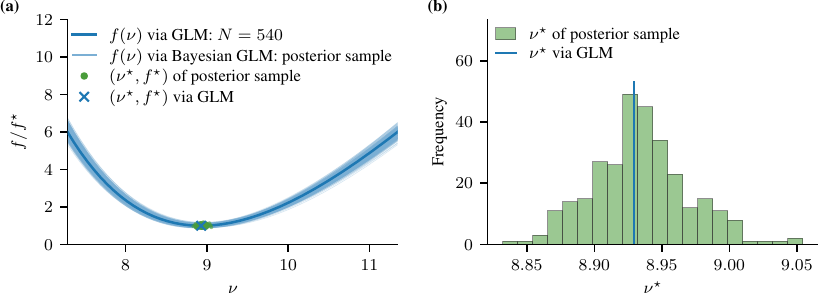}
\caption{Bayesian GLM surrogate for the calibration of $\nu$ on ClimaX (calibration year 2016). \textbf{(a)} The surrogate objective $f(\nu)$, the expectation in Eq.~\eqref{eq:bo_surrogate_loss}, normalized by its minimum $f^\star$, from the GLM fitted to all 540 evaluations (thick) and from 300 posterior samples of the Bayesian GLM (thin), with the minimizer of each. \textbf{(b)} The minimizers of the posterior samples lie between $8.83$ and $9.05$ and all round to $\nu^\star=9$; the line is the GLM estimate, $8.93$.}
\label{fig:bo_sensitivity}
\end{figure}

\subsection{Scope of $\nu$: Per-Layer and Per-Head Concentrations}
\label{app:nu-scope}

\begin{table}[tp]
\centering
\caption{Scope of $\nu$ on ClimaX (single backbone, calibration year 2016, $S{=}19$). SA with one $\nu$ selected by BO, and the relaxation of Eq.~\eqref{eq:sa_relaxation} fitted by Adam on $\mathcal{L}_{\mathrm{SA}}$ with one $\nu$, one per layer, and one per head; the $\nu$ column gives the fitted value or the minimum, median, and maximum of the fitted values. $W_1$ on all 17{,}376 test forecasts and on their JJA (4{,}272) and DJF (4{,}176) subsets; MPIW (width of the $[\min,\max]$ range, $\mathrm{m^2s^{-2}}$) and ACC (of the 19-sample mean) on all test forecasts. Every metric is area-weighted. Fit time on one GPU. Bold: best; underlined: second best.}
\label{tab:nu_scope}
\footnotesize
\setlength{\tabcolsep}{3pt}
\begin{tabular}{@{}lcccccccc@{}}
\toprule
& & & & \multicolumn{3}{c}{$W_1\downarrow$} & & \\
\cmidrule(lr){5-7}
Calibration & \#$\nu$ & $\nu$ & Fit & full & JJA & DJF & MPIW$\downarrow$ & ACC$\uparrow$ \\
\midrule
BO, global & 1 & 9 & 3\,min & \textbf{0.0396} & \underline{0.0381} & \textbf{0.0411} & \textbf{711} & \textbf{0.9541} \\
Adam, global & 1 & 9 & 18\,min & \underline{0.0398} & \underline{0.0381} & 0.0415 & \textbf{711} & \textbf{0.9541} \\
Adam, per layer & 8 & 8/9/10 & 44\,min & \textbf{0.0396} & \textbf{0.0379} & \textbf{0.0411} & \underline{714} & \underline{0.9539} \\
Adam, per head & 128 & 5/8/14 & 33\,min & 0.0401 & 0.0388 & \underline{0.0414} & 715 & \underline{0.9539} \\
\bottomrule
\end{tabular}
\end{table}

SA uses one concentration $\nu$ for every attention row, while the shape of its dispersion already varies with the input, head, and layer (Proposition~\ref{prop:variance}). To test whether separate scales help, we fit $\nu$ on ClimaX at three scopes: one global value, one per layer (8), and one per head (128). Integer $\nu$ admits no gradient, so we relax SA to continuous $\nu\ge1$ by interpolating between adjacent sample averages. With $n=\lfloor\nu\rfloor$, $r=\nu-n$, and $\bar{\mathbf{z}}_k=\tfrac1k\sum_{m=1}^{k}\mathbf{e}_{z_t^{(m)}}$ the average of the first $k$ of $n+1$ categorical draws from $\boldsymbol{\pi}_t$,
\begin{equation}
\widetilde{\boldsymbol{\pi}}_t(\nu)=\bigl(1-\rho(r)\bigr)\,\bar{\mathbf{z}}_n+\rho(r)\,\bar{\mathbf{z}}_{n+1},
\qquad \rho(r)=3r^2-2r^3 .
\label{eq:sa_relaxation}
\end{equation}
Because $\rho(r)\in[0,1]$, $\widetilde{\boldsymbol{\pi}}_t(\nu)$ is a convex combination of one-hot vectors and so a probability vector. It equals the SA weights of Definition~\ref{def:stochastic_attention_weights} at integer $\nu$, and it is unbiased, $\mathbb{E}[\widetilde{\boldsymbol{\pi}}_t(\nu)\mid\boldsymbol{\pi}_t]=\boldsymbol{\pi}_t$, for every $\nu$, so Proposition~\ref{prop:unbiased} still holds. Its covariance, $\bigl[\tfrac1n-\tfrac{\rho(r)(2-\rho(r))}{n(n+1)}\bigr]\bigl(\diag(\boldsymbol{\pi}_t)-\boldsymbol{\pi}_t\boldsymbol{\pi}_t^\intercal\bigr)$, decreases monotonically from its value at $\nu=n$ to its value at $\nu=n+1$ in Proposition~\ref{prop:variance}. With the draws held fixed, $\widetilde{\boldsymbol{\pi}}_t$ is differentiable in $\nu$, with $\partial\widetilde{\boldsymbol{\pi}}_t/\partial\nu=\rho'(r)(\bar{\mathbf{z}}_{n+1}-\bar{\mathbf{z}}_n)$, which vanishes at integers since $\rho'(0)=\rho'(1)=0$. We parameterize each concentration as $\nu=1+\operatorname{softplus}(\theta)$ and minimize the same objective $\mathcal{L}_{\mathrm{SA}}$ (Eq.~\eqref{eq:lsa}) on the same calibration year with Adam (learning rate $0.05$, four stochastic passes per step, 200 calibration examples per epoch, at most 50 epochs, $\nu$ initialized at $4.5$, away from the integers where the gradient vanishes). The fitted values settle at or near integers.

Table~\ref{tab:nu_scope} compares the three fits with SA calibrated by BO. The global fit converges to $\nu=9.0$, where the relaxation coincides with SA, so its row and the BO row are two independent draws of the same predictor; their $W_1$ differ by at most $0.0003$. The per-layer values span 8--10 and the per-head values 5--14, so the heads do prefer different scales, but the test $W_1$ of either fit stays within $0.0007$ of BO's on all test forecasts and on both seasons, and its intervals are at most $0.5\%$ wider. Fitting per layer or per head takes 33--44 GPU-minutes instead of BO's 3 minutes. A single global $\nu$ is therefore adequate.

\section{Experimental Setup and Reproducibility Details}
\label{app:exp-details}

This appendix section summarizes the experimental settings used in the paper and provides the key implementation details needed to reproduce the reported uncertainty diagnostics. We consider three evaluation settings: (i) transformer-based time-series forecasting with \textsc{TimesFM}, (ii) global weather forecasting with \textsc{ClimaX}, and (iii) UCI regression as a compact testbed for evaluating calibration and sharpness beyond the scientific forecasting examples. Baseline-method summaries are given separately in Appendix~\ref{app:baselines}.

\subsection{TimesFM: Architecture and Forecasting}
\paragraph{Backbone.}
\textsc{TimesFM} is a decoder-only transformer for time-series forecasting that operates on non-overlapping input patches \cite{Das2024}. Each patch is mapped to a token representation and processed by stacked causal self-attention layers. Forecasts are decoded in output patches and extended autoregressively for longer horizons. For deterministic reporting we use the point-forecast decoding head (decode index $=5$).

\paragraph{Datasets and splits.}
We evaluate on the ETT-small benchmarks ETTh1, ETTh2, ETTm1, and ETTm2 \cite{Zhou2021}, using the standard Informer split boundaries:
ETTh*: training 0--8640, validation 8640--11520, test 11520--14400;
ETTm*: training 0--34560, validation 34560--46080, test 46080--57600.
Normalization statistics are computed on the full training segment, while model adaptation uses only the first 10\% of the training portion. Each column is treated as an independent univariate series.

\paragraph{Calibration data.}
SA's $\nu$ (one per series) and the baselines' per-series temperatures are fitted on windows drawn from the validation segment alone, which is held out from fine-tuning and lies between training and test.

\paragraph{Forecasting protocol and metrics.}
We use context length 512 and forecasting horizons of 96 and 192. Test windows are drawn at stride 24 (665 windows over the seven series of ETTh1 at $H{=}96$), with $S{=}100$ predictive draws per method. MAE is averaged across series and forecast horizon. $W_1$ is computed from the rank histogram of the 100 draws (101 bins); coverage is that of the central 95\% interval of the draws, and MPIW is its mean width, in the units of the series standardized with the training statistics, like MAE.

\paragraph{Lightweight task adaptation.}
We use LoRA adapters attached to the query--key--value projection, attention output projection, and the two linear layers in the feed-forward block. All pretrained backbone parameters are frozen and only LoRA parameters are trained.

\subsection{ClimaX: Architecture and Global Forecasting Protocol}
\paragraph{Backbone.}
\textsc{ClimaX} is a ViT-style foundation model for gridded weather and climate fields \cite{NguyenClimaX23}. It tokenizes each variable field into spatial patches, aggregates variable information through cross-attention, and processes the resulting token sequence with a ViT backbone.

\paragraph{Task and metrics.}
We evaluate global forecasting at a 72-hour lead time. Accuracy is reported as the latitude-weighted anomaly correlation coefficient (ACC) of the deterministic prediction and of the mean of the 19 samples.

\paragraph{Model adaptation.}
We start from the pretrained 5.625$^\circ$ checkpoint and use LoRA adapters on the attention projections and MLP linear layers. The pretrained backbone is frozen, while normalization layers remain trainable.

\paragraph{Compute.}
The costs reported in Table~\ref{tab:climax_summary} are wall-clock times on the NVIDIA V100-SXM2-32GB GPU.

\subsection{UCI Regression Benchmarks}
\label{app:uci-protocol}
\paragraph{Backbone.}
We use a deterministic FT-Transformer regressor \cite{Gorishniy2021}, trained from scratch on each split, with token dimension 64 to 192, 4 or 8 heads, and 2 or 3 layers, set per dataset. SA and MC Dropout draw from this backbone, while SWAG, MultiSWAG, HSA, and IVON train their own models on the same splits.

\paragraph{Datasets and splits.}
We use the eight datasets and the split files of \citet{Gal2016}: 20 random 90/10 splits per dataset, 5 for Protein, with the last 20\% of each training pool as the validation fold.

\paragraph{Calibration data.}
SA's $\nu$, fitted with the routine of Appendix~\ref{app:bo-nu}, and each baseline's temperature are fitted on the split's validation fold, which also sets early stopping where a method uses it.

\paragraph{Protocol and metrics.}
Every method has $S{=}99$ predictive samples per test point. $W_1$ is computed from the rank histogram of the samples (100 bins) and pooled over a dataset's splits, because a small test set inflates its own value: on Yacht, with 31 test points per split, a calibrated ensemble has expected $W_1$ of $0.056$ per split and $0.013$ pooled. The 95\% intervals on pooled $W_1$ come from a cluster bootstrap over splits. Coverage and width use the interval between the 5th and 95th of the 99 ordered samples, whose expected coverage under calibration is exactly $0.90$ \cite{Wilks1941}, the nominal level of ClimaX's $[\min,\max]$ range.

\paragraph{Statistical tests.}
Because the splits share about 90\% of their training data, a paired comparison on per-split values counts as significant only when the corrected resampled $t$-test of \citet{Nadeau2003} and the exact Wilcoxon signed-rank test \cite{Wilcoxon1945} agree after Holm's correction \cite{Holm1979} within each dataset, two-sided at the $0.05$ level. Protein's five splits cannot reach this level under the Wilcoxon test, so there the corrected $t$-test and the same sign on all five splits decide.

\subsection{Predictive Samples for Calibration Diagnostics}
All uncertainty methods considered in the paper induce predictive distributions through samples. To unify evaluation across methods, we form an empirical predictive distribution $\widehat{F}$ from repeated predictive samples for each test case. PIT values are then computed from $\widehat{F}$ for scalar outcomes, and coordinate-wise for structured outputs, aggregating over coordinates when needed.

\paragraph{Stochastic attention.}
For stochastic attention, each predictive sample corresponds to a single stochastic forward pass of the frozen model, with attention rows sampled according to Algorithm~\ref{alg:stochastic_attention}. The concentration parameter $\nu$ is calibrated on held-out data using Algorithm~\ref{alg:calibrate_nu_appendix}, after which repeated stochastic forward passes are used to estimate calibration and sharpness diagnostics.

\paragraph{SWAG-LoRA and MultiSWAG.}
For SWAG-LoRA, predictive samples are obtained by sampling LoRA weights from the SWAG approximation and performing one forward pass per draw. For MultiSWAG, samples are aggregated across multiple independently trained SWAG-LoRA models.

\paragraph{IVON-LoRA.}
For IVON-LoRA, predictive samples are generated by sampling LoRA weights from the learned Gaussian variational approximation and running one forward pass per sample.

\subsection{Implementation Hyperparameters}

\paragraph{TimesFM.}
We use \textsc{TimesFM} 2.5 with 20 layers, 16 heads, model dimension 1280, input patch length 32, and output patch length 128. LoRA uses rank 8 and $\alpha=8$. Optimization uses Adam with cosine learning rate decay, a batch size of 16, and 50 epochs.

\paragraph{ClimaX.}
We use the pretrained 5.625$^\circ$ \textsc{ClimaX} checkpoint with a patch size of 2, an embedding dimension of 1024, a depth of 8, and 16 attention heads. LoRA uses rank 8 and $\alpha=8$. Optimization uses AdamW with warmup and cosine decay in fp16 DDP training.

\paragraph{Sampling budgets.}
Unless otherwise stated, predictive diagnostics are computed from repeated Monte Carlo samples from each method. In all cases, the same sampling budget is used within a given experiment when comparing methods.

\subsection{Baseline Method Summaries}
\label{app:baselines}

This subsection summarizes the baseline uncertainty methods used in the experiments. The baselines were selected to represent different mechanisms for producing predictive uncertainty in fine-tuned transformer models: stochastic forward passes induced by dropout, weight-space posterior approximations, mixtures of posterior approximations, variational/Bayesian training, and stochastic architectural perturbations. This organization is useful because stochastic attention is best understood not only as another competing uncertainty method, but as a method that occupies a different point in the design space of \emph{where randomness is introduced} and \emph{when calibration is performed}.

\subsubsection{MC Dropout}
\label{app:mcdropout}

MC dropout interprets dropout training as approximate Bayesian inference and obtains predictive uncertainty by keeping dropout active at test time and averaging repeated stochastic forward passes \cite{Gal2016}. At test time, MC dropout performs $S$ stochastic forward passes with independently sampled masks. In our comparisons, MC dropout serves as a low-cost stochastic-inference baseline.

\subsubsection{Contextual Dropout}
\label{app:contextual_dropout}

Contextual dropout generalizes standard dropout by making the dropout distribution depend on the input covariates \cite{Fan2021}. For each training pair $(\mathbf{x}_i,\mathbf{y}_i)$, contextual dropout introduces a variational distribution $q_{\boldsymbol{\phi}}(\mathbf{z}_i \mid \mathbf{x}_i)$, and optimizes the sample-wise evidence lower bound. Unlike MC dropout, contextual dropout does not rely on fixed hand-tuned dropout rates; instead, it learns sample-dependent stochasticity during training.

\subsubsection{Stochastic Weight Averaging Gaussian (SWAG)}
\label{app:swag}

SWAG \cite{Maddox2019} builds a Gaussian approximation to the weight posterior using the trajectory of SGD iterates, building on SWA \cite{Izmailov2018}. The resulting SWAG posterior approximation is
\[
q_{\text{SWAG}}(\boldsymbol{\theta}) = \mathcal{N}\!\Big(\boldsymbol{\theta}_{\text{SWA}}, \tfrac{1}{2}(\boldsymbol{\Sigma}_{\diag} + \boldsymbol{\Sigma}_{\text{lr}})\Big),
\]
where $\boldsymbol{\Sigma}_{\diag}$ and $\boldsymbol{\Sigma}_{\text{lr}}$ are diagonal and low-rank covariance terms estimated from the SGD trajectory.
We did not run SWAG-SGLD: like SWAG, it builds its approximation from a weight trajectory collected during training, so it occupies the same cost regime as SWAG and MultiSWAG.

\subsubsection{MultiSWAG}
\label{app:multiswag}

MultiSWAG \cite{Wilson2020} combines multiple independently trained SWAG approximations into a mixture-of-Gaussians posterior surrogate. Each mixture component $q_l(\boldsymbol{\theta})$ is an independently trained SWAG posterior, and the predictive distribution is approximated by sampling and averaging across all components. Relative to a single SWAG fit, this requires $L$ independent SWAG runs.

\subsubsection{Improved Variational Online Newton (IVON)}
\label{app:ivon}

IVON is a variational-learning optimizer that learns a diagonal Gaussian distribution over the weights while training with an Adam-like loop \cite{Shen2024}. Unlike post-hoc weight-space methods such as SWAG, IVON optimizes the variational objective directly during training and uses the learned weight distribution itself for uncertainty estimation.

\subsubsection{Low-Rank Adaptation (LoRA)}
\label{app:lora}

LoRA fine-tunes a pre-trained model by freezing its original weights and introducing a low-rank parameterization for the weight update in selected linear layers \citep{Hu2022}. For a pre-trained weight matrix $\mathbf{W}_0 \in \mathbb{R}^{d \times k}$, LoRA reparameterizes the adapted weight as $\mathbf{W} = \mathbf{W}_0 + \mathbf{B}\mathbf{A}$, where $\mathbf{B} \in \mathbb{R}^{d \times r}$ and $\mathbf{A} \in \mathbb{R}^{r \times k}$ with rank $r \ll \min(d,k)$.

\subsubsection{SWAG-LoRA, MultiSWAG-LoRA, and IVON-LoRA}
\label{app:swag-multiswag-ivon-lora}

We implement SWAG-LoRA and MultiSWAG-LoRA by combining the SWAG and MultiSWAG procedures with LoRA \cite{Hu2022,Izmailov2018,Maddox2019,Onal2024}. The posterior approximation is placed only over the trainable LoRA adapters rather than over the full set of model weights. IVON-LoRA replaces the standard LoRA optimizer with the IVON variational-learning procedure on the LoRA parameters, while keeping the backbone fixed \citep{Shen2024,Cong2024}.

\subsubsection{Hierarchical Stochastic Attention (HSA)}
\label{app:hsa}

Hierarchical stochastic attention (HSA) injects randomness directly into transformer self-attention through Gumbel-Softmax sampling \cite{Pei2022}. HSA injects randomness at two levels: first, through stochastic assignment of keys to learned centroids, and second, through stochastic attention over values. In our experiments, HSA serves as the closest architectural baseline to stochastic attention, since both methods alter the attention pathway itself.

\section{Additional ClimaX Results}
\label{app:climax-additional}

\subsection{Statistical Study: Protocol and Paired Tests}
\label{app:climax-stats}

We repeat the whole ClimaX pipeline 20 times. Each run re-fine-tunes the backbone from the pretrained checkpoint with its own seed, draws its calibration year from 1979--2016 and fine-tunes on the remaining years (run 0 is the standard split with calibration year 2016), keeps the test years 2017--2018 fixed, and refits $\nu$ or $T$. MultiSWAG, which needs ten trainings per run, has 8 runs. Runs share splits and seeds across methods, so comparisons are paired; we report the paired $t$-test and the Wilcoxon signed-rank test \cite{Wilcoxon1945}. Across the runs SA's fitted $\nu^\star$ is $10.55\pm1.39$ (range 7--12).

Table~\ref{tab:climax_paired} gives the paired differences behind the significance statements of Section~\ref{sec:results_forecast}. Within each method the run-to-run variation is small relative to the differences between methods. Figure~\ref{fig:climax_sharpness} shows how the interval widths are distributed over the test forecasts and grid points of run 0; SA's are the narrowest at the median and in the upper tail.

\begin{table}[tp]
\centering
\caption{Paired differences, SA minus baseline, over the shared runs ($n$): mean difference, paired $t$-test $p$, and Wilcoxon signed-rank $p$. $W_1$ compares SA's native value with the baseline's temperature-scaled value; MPIW is at the operating points; ACC is of the 19-sample mean. All three are area-weighted. Every MPIW difference is negative in every run (Wilcoxon statistic 0).}
\label{tab:climax_paired}
\footnotesize
\setlength{\tabcolsep}{2.5pt}
\begin{tabular}{@{}lc ccc ccc ccc@{}}
\toprule
& & \multicolumn{3}{c}{$\Delta W_1$} & \multicolumn{3}{c}{$\Delta$MPIW} & \multicolumn{3}{c}{$\Delta$ACC} \\
\cmidrule(lr){3-5}\cmidrule(lr){6-8}\cmidrule(lr){9-11}
Baseline & $n$ & mean & $p_t$ & $p_W$ & mean & $p_t$ & $p_W$ & mean & $p_t$ & $p_W$ \\
\midrule
SWAG & 20 & $-0.0213$ & $6{\times}10^{-9}$ & $2{\times}10^{-6}$ & $-368$ & $4{\times}10^{-16}$ & $2{\times}10^{-6}$ & $-0.0121$ & $4{\times}10^{-12}$ & $2{\times}10^{-6}$ \\
MultiSWAG & 8 & $-0.0207$ & $1{\times}10^{-4}$ & 0.008 & $-304$ & $1{\times}10^{-7}$ & 0.008 & $-0.0118$ & $1{\times}10^{-6}$ & 0.008 \\
IVON & 20 & $-0.0591$ & $6{\times}10^{-16}$ & $2{\times}10^{-6}$ & $-738$ & $3{\times}10^{-19}$ & $2{\times}10^{-6}$ & $+0.0182$ & $2{\times}10^{-15}$ & $2{\times}10^{-6}$ \\
Contextual Dropout & 20 & $-0.0219$ & $5{\times}10^{-10}$ & $2{\times}10^{-6}$ & $-260$ & $4{\times}10^{-19}$ & $2{\times}10^{-6}$ & $-0.0092$ & $4{\times}10^{-10}$ & $2{\times}10^{-6}$ \\
HSA & 20 & $+0.0062$ & $9{\times}10^{-4}$ & $9{\times}10^{-4}$ & $-1781$ & $6{\times}10^{-29}$ & $2{\times}10^{-6}$ & $+0.1733$ & $4{\times}10^{-31}$ & $2{\times}10^{-6}$ \\
\bottomrule
\end{tabular}
\end{table}

\begin{figure}[tp]
\centering
\includegraphics{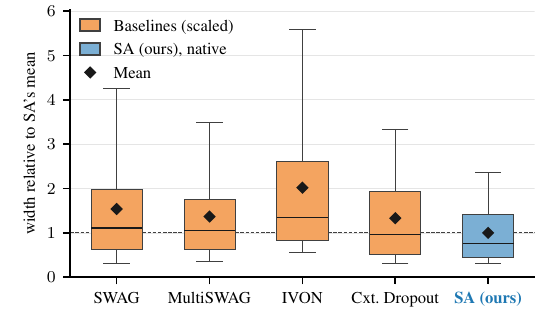}
\caption{Distribution of interval widths on ClimaX in run 0 of the statistical study (calibration year 2016), over every test forecast and grid point, relative to SA's mean width. The interval is the $[\min,\max]$ range of the 19 samples, a nominal 90\% interval, at the operating points, with the baselines temperature-scaled and SA native. Boxes are the quartiles, whiskers the 5th and 95th percentiles, and diamonds the means, all area-weighted. HSA, whose mean width is $3.4\times$ SA's, is left out to keep the scale readable.}
\label{fig:climax_sharpness}
\end{figure}

\subsection{Held-out-Season Shift}
\label{app:climax-shift}

To test a regime absent from calibration, we hold one season out of the calibration year 2016 and evaluate on that season of the test years, in both directions: JJA, with $\nu$ and $T$ fitted on September--May and tested on June--August 2017--18 (4{,}272 forecasts), and DJF, fitted on March--November and tested on December--February (4{,}176). The boundaries are trimmed so that no held-out-season target enters calibration through the 72-hour lead. Table~\ref{tab:climax_shift_full} reports every method natively, with a temperature fitted on the full calibration year, and with a temperature fitted with the test season held out, on the held-out season and on its most extreme decile (the top 10\% of forecasts by latitude-weighted $|Z_{500}$ anomaly$|$: 427 JJA and 418 DJF forecasts). All conditions are scored on identical example sets.

Without their correction the baselines remain severely miscalibrated on the held-out season, so their competitiveness there is entirely post-hoc; SA needs no correction and remains the sharpest in both directions by a wide margin. On JJA its $W_1$ ($0.038$) is the lowest of all methods, and on the most extreme decile of JJA forecasts it stays at $W_1{=}0.044$. Two results do not favor SA. Holding JJA out leaves $\nu^\star{=}9$ unchanged, but holding DJF out moves it to 12, which narrows the DJF intervals ($725\to645$) at a cost in calibration ($W_1$ $0.041\to0.053$), so that on DJF scaled HSA has the lower $W_1$ ($0.041$), with intervals $4.2\times$ as wide. And the baselines' temperatures barely move when a season is removed ($|\Delta T|\le1.4\%$, $|\Delta W_1|\le0.0012$): on global $Z_{500}$ a refit temperature does not go stale under seasonal shift. This experiment therefore supports that SA needs no output-level correction and stays far sharper under shift, not that its calibration transfers better than a refit temperature.

\begin{table}[tp]
\centering
\caption{ClimaX with one season held out of calibration (single backbone, calibration year 2016, $S{=}19$). Each method is scored natively, with a temperature fitted on the full calibration year, and with a temperature fitted with the test season held out (September--May for JJA, March--November for DJF), on the held-out season and on its most extreme decile; SA is native at its own $\nu^\star$ (9 on the full year and with JJA held out, 12 with DJF held out). MPIW in $\mathrm{m^2s^{-2}}$; every metric is area-weighted. Bold: best; underlined: second best, among the season-held-out rows (SA native, baselines temperature-scaled).}
\label{tab:climax_shift_full}
\small
\setlength{\tabcolsep}{3pt}
\resizebox{\linewidth}{!}{%
\begin{tabular}{@{}llcccccccc@{}}
\toprule
& & \multicolumn{2}{c}{JJA} & \multicolumn{2}{c}{JJA extremes} & \multicolumn{2}{c}{DJF} & \multicolumn{2}{c}{DJF extremes} \\
\cmidrule(lr){3-4}\cmidrule(lr){5-6}\cmidrule(lr){7-8}\cmidrule(lr){9-10}
Method & Setting & $W_1$ & MPIW & $W_1$ & MPIW & $W_1$ & MPIW & $W_1$ & MPIW \\
\midrule
SWAG & native & 0.2270 & 35 & 0.2289 & 35 & 0.2255 & 37 & 0.2299 & 39 \\
 & TS, full year & 0.0622 & 1020 & 0.0603 & 1030 & 0.0622 & 1106 & 0.0692 & 1165 \\
 & TS, season held out & 0.0619 & 1011 & 0.0601 & 1020 & 0.0620 & 1103 & 0.0691 & 1162 \\
\midrule
MultiSWAG & native & 0.2266 & 39 & 0.2284 & 39 & 0.2250 & 41 & 0.2296 & 44 \\
 & TS, full year & 0.0728 & 943 & 0.0697 & 955 & 0.0710 & 1005 & 0.0779 & 1064 \\
 & TS, season held out & 0.0728 & \underline{941} & 0.0697 & \underline{953} & 0.0709 & 1000 & 0.0780 & 1058 \\
\midrule
IVON & native & 0.2368 & 103 & 0.2443 & 102 & 0.2197 & 109 & 0.2308 & 109 \\
 & TS, full year & 0.0948 & 1330 & 0.1031 & 1324 & 0.0738 & 1415 & 0.0917 & 1404 \\
 & TS, season held out & 0.0956 & 1314 & 0.1039 & 1309 & 0.0739 & 1412 & 0.0919 & 1401 \\
\midrule
Contextual Dropout & native & 0.1215 & 418 & 0.1302 & 422 & 0.1128 & 447 & 0.1263 & 475 \\
 & TS, full year & 0.0737 & 937 & 0.0817 & 948 & 0.0640 & 1002 & 0.0806 & 1066 \\
 & TS, season held out & 0.0734 & 943 & 0.0814 & 954 & 0.0644 & \underline{991} & 0.0810 & \underline{1054} \\
\midrule
HSA & native & 0.1867 & 317 & 0.1878 & 329 & 0.1814 & 368 & 0.1836 & 395 \\
 & TS, full year & 0.0429 & 2318 & 0.0372 & 2406 & 0.0398 & 2688 & 0.0426 & 2891 \\
 & TS, season held out & \underline{0.0428} & 2305 & \textbf{0.0371} & 2393 & \textbf{0.0410} & 2727 & \textbf{0.0437} & 2933 \\
\midrule
SA (ours) & native, full year & 0.0381 & 682 & 0.0436 & 703 & 0.0411 & 725 & 0.0480 & 749 \\
 & native, season held out & \textbf{0.0381} & \textbf{682} & \underline{0.0436} & \textbf{703} & \underline{0.0533} & \textbf{645} & \underline{0.0614} & \textbf{665} \\
\bottomrule
\end{tabular}%
}
\end{table}

\subsection{Per-Pass Latency}
\label{app:latency}

Per-forward-pass latency of ClimaX increases gradually with $\nu$: $0.457$\,s at $\nu{=}1$, $0.471$\,s at $\nu{=}16$, and $0.529$\,s at $\nu{=}512$, against $0.424$\,s for deterministic inference (Figure~\ref{fig:latency_appendix}). Over the calibrated range $\nu^\star\in[7,12]$ the overhead is therefore at most $11\%$.

\begin{figure}[tp]
\centering
\includegraphics{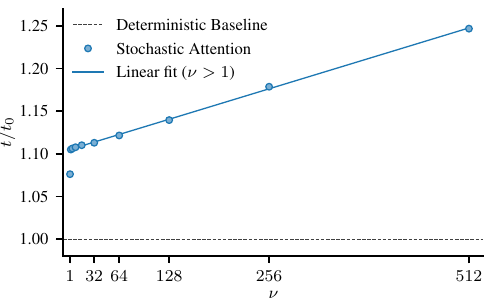}
\caption{Per-forward-pass latency on ClimaX as a function of $\nu$, relative to the deterministic pass.}
\label{fig:latency_appendix}
\end{figure}

\subsection{SWAG Variance-Scale Sweep}
\label{app:swag-sweep}

A natural objection to the calibration comparison is that SWAG might simply be poorly tuned. We test this directly on ClimaX by sweeping the SWAG variance scale over $\{1, 2, 4, 8, 16, 32, 64\}$ on the trained checkpoint, holding everything else fixed. Table~\ref{tab:swag_sweep_val} reports validation results across all scales, and Table~\ref{tab:swag_sweep_test} reports held-out test results at the two operating points highlighted by the sweep. The sweep reveals an irreducible trade-off: increasing the scale improves calibration and coverage but degrades accuracy and widens intervals. Even the best natively calibrated SWAG (scale~64) reaches only $W_1{=}0.116$ at coverage $60.9\%$, still over $2{\times}$ SA's $W_1$ at comparable coverage; the scale that retains the best accuracy (scale~16) is even less calibrated ($W_1{=}0.175$). The native under-dispersion of SWAG on ClimaX is not a tuning artifact.

\begin{table}[tp]
\centering
\caption{Native SWAG variance-scale sweep on ClimaX validation data. Increasing the scale improves calibration and coverage, initially improves accuracy, then widens intervals; scale 64 is selected by validation $W_1$, scale 16 is the best accuracy/sharpness compromise. \textbf{Bold} marks the best value per column and denote best sharpness  relative to the best RMSE.}
\label{tab:swag_sweep_val}
\small
\begin{tabular}{@{}ccccc@{}}
\toprule
Scale & Val.\ RMSE $\downarrow$ & Val.\ $W_1$ $\downarrow$ & Val.\ Cov@95 $\uparrow$ & Val.\ Sharp@95 $\downarrow$ \\
\midrule
1  & 265.73 & 0.215 & 0.089 & 47.4 \\
2  & 265.43 & 0.210 & 0.123 & 64.9 \\
4  & 265.01 & 0.202 & 0.170 & 89.2 \\
8  & 264.45 & 0.190 & 0.233 & 123.0 \\
16 & \textbf{263.86} & 0.174 & 0.320 & \textbf{171.3} \\
32 & 263.93 & 0.149 & 0.445 & 247.7 \\
64 & 267.44 & \textbf{0.116} & \textbf{0.604} & 379.6 \\
\bottomrule
\end{tabular}
\end{table}

\begin{table}[tp]
\centering
\caption{Held-out test results for the two SWAG operating points highlighted by the variance sweep. Scale 64 is the fairest best-calibrated SWAG point under the validation $W_1$ rule; scale 16 is the strongest accuracy/sharpness compromise.}
\label{tab:swag_sweep_test}
\small
\begin{tabular}{@{}lcccc@{}}
\toprule
Test operating point & RMSE $\downarrow$ & $W_1$ $\downarrow$ & Cov@95 $\uparrow$ & Sharp@95 $\downarrow$ \\
\midrule
SWAG @ 16 (compromise) & \textbf{263.71} & 0.175 & 0.321 & \textbf{171.0} \\
SWAG @ 64 (val.-selected) & 266.67 & \textbf{0.116} & \textbf{0.609} & 381.2 \\
\bottomrule
\end{tabular}
\end{table}

\section{TimesFM Details}
\label{app:timesfm-details}

Table~\ref{tab:timesfm} compares the six methods on all eight ETT configurations, with SA native, its $\nu$ fitted on the validation segment, and every baseline temperature-scaled per series on the same segment. Figure~\ref{fig:timesfm_scaled_appendix} shows the window of Figure~\ref{fig:timesfm_diagnostic} with the baselines temperature-scaled.

\begin{table}[tp]
\centering
\caption{TimesFM on all eight ETT configurations ($S{=}100$), with SA native and every baseline temperature-scaled per series on the validation segment; the last two columns are the mean and the worst case over the configurations. MAE is that of the sample mean and MPIW the mean width of the central 95\% interval, both in the units of the series standardized with the training statistics. The one-time preparation on ETTh1 at $H{=}96$ takes 3 minutes for SA, against 5 for SWAG, 47 for MultiSWAG, 9 for IVON, 19 for Contextual Dropout and 7 for HSA. Bold: best; underlined: second best, per column.}
\label{tab:timesfm}
\footnotesize
\setlength{\tabcolsep}{3pt}
\begin{tabular}{@{}lcccccccccc@{}}
\toprule
& \multicolumn{2}{c}{ETTh1} & \multicolumn{2}{c}{ETTh2} & \multicolumn{2}{c}{ETTm1} & \multicolumn{2}{c}{ETTm2} & & \\
\cmidrule(lr){2-3}\cmidrule(lr){4-5}\cmidrule(lr){6-7}\cmidrule(lr){8-9}
Method & $H{=}96$ & 192 & 96 & 192 & 96 & 192 & 96 & 192 & mean & worst \\
\midrule
\multicolumn{11}{@{}l}{$W_1\downarrow$} \\
SWAG & 0.0876 & 0.0804 & 0.0723 & 0.0663 & 0.0564 & 0.0789 & 0.0632 & 0.0665 & 0.0714 & 0.0876 \\
MultiSWAG & 0.0828 & 0.0784 & 0.0697 & 0.0653 & 0.0596 & 0.0715 & 0.0674 & 0.0701 & 0.0706 & 0.0828 \\
IVON & \underline{0.0408} & 0.0518 & \textbf{0.0287} & 0.0483 & \underline{0.0279} & \underline{0.0381} & \underline{0.0489} & 0.0737 & \underline{0.0448} & 0.0737 \\
Contextual Dropout & 0.0659 & \underline{0.0395} & 0.0409 & 0.0501 & 0.0511 & 0.0577 & 0.0617 & 0.0704 & 0.0547 & \underline{0.0704} \\
HSA & 0.0497 & 0.0894 & \underline{0.0376} & \textbf{0.0385} & 0.0681 & 0.0845 & 0.0564 & \underline{0.0555} & 0.0600 & 0.0894 \\
SA (ours) & \textbf{0.0292} & \textbf{0.0371} & 0.0470 & \underline{0.0466} & \textbf{0.0234} & \textbf{0.0352} & \textbf{0.0302} & \textbf{0.0276} & \textbf{0.0345} & \textbf{0.0470} \\
\midrule
\multicolumn{11}{@{}l}{MAE$\downarrow$} \\
SWAG & 0.4516 & 0.4785 & 0.3930 & 0.4461 & 0.4157 & 0.4524 & 0.2939 & 0.3428 & 0.4093 & 0.4785 \\
MultiSWAG & 0.4471 & 0.4764 & 0.3930 & 0.4484 & 0.4023 & 0.4534 & 0.2900 & 0.3397 & 0.4063 & 0.4764 \\
IVON & \textbf{0.4009} & \textbf{0.4295} & 0.4014 & 0.4415 & 0.5730 & 0.5790 & 0.3021 & \textbf{0.3176} & 0.4306 & 0.5790 \\
Contextual Dropout & 0.4472 & \underline{0.4337} & \textbf{0.3731} & \textbf{0.4264} & \underline{0.3885} & \underline{0.4486} & \underline{0.2874} & 0.3436 & \underline{0.3936} & \underline{0.4486} \\
HSA & 0.4640 & 0.5369 & 0.3929 & 0.4341 & 0.5122 & 0.6027 & 0.3271 & 0.3675 & 0.4547 & 0.6027 \\
SA (ours) & \underline{0.4289} & 0.4439 & \underline{0.3794} & \underline{0.4312} & \textbf{0.3874} & \textbf{0.4294} & \textbf{0.2742} & \underline{0.3190} & \textbf{0.3867} & \textbf{0.4439} \\
\midrule
\multicolumn{11}{@{}l}{MPIW$\downarrow$} \\
SWAG & 4.513 & 5.534 & 7.089 & 8.698 & 2.994 & 4.181 & 4.127 & 5.548 & 5.336 & 8.698 \\
MultiSWAG & 3.462 & 3.956 & 5.206 & 6.154 & 2.587 & 3.344 & 6.631 & 8.017 & 4.920 & 8.017 \\
IVON & \underline{2.491} & 3.101 & \underline{2.136} & 3.384 & 3.011 & 3.366 & \underline{1.781} & 3.905 & 2.897 & 3.905 \\
Contextual Dropout & 2.623 & \underline{2.711} & 2.599 & 3.424 & \underline{2.442} & \underline{3.134} & 2.747 & \underline{3.470} & \underline{2.894} & \underline{3.470} \\
HSA & 2.518 & 3.529 & 2.512 & \underline{2.873} & 3.903 & 5.930 & 3.101 & 3.514 & 3.485 & 5.930 \\
SA (ours) & \textbf{1.747} & \textbf{1.755} & \textbf{1.202} & \textbf{1.396} & \textbf{1.578} & \textbf{1.588} & \textbf{0.949} & \textbf{1.175} & \textbf{1.424} & \textbf{1.755} \\
\bottomrule
\end{tabular}
\end{table}

\begin{figure}[tp]
\centering
\includegraphics{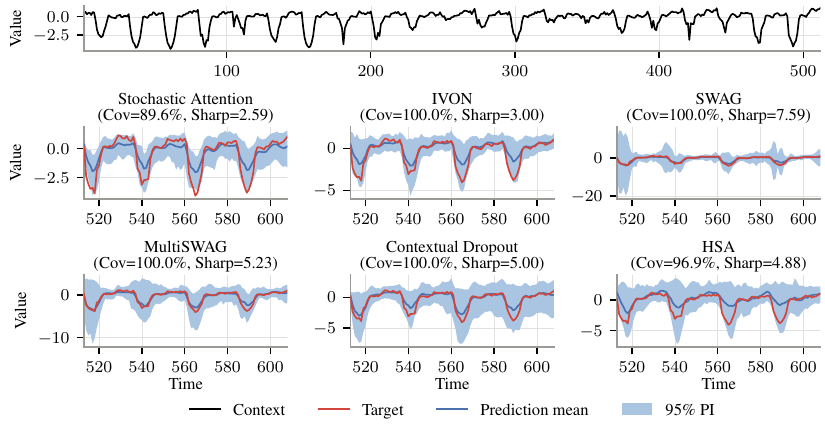}
\caption{The window of Figure~\ref{fig:timesfm_diagnostic} with every baseline temperature-scaled per series on the validation segment (SA native). Scaling lifts the baselines' coverage of this window to 97--100\%, with intervals $1.2\times$ (IVON) to $2.9\times$ (SWAG) as wide as SA's.}
\label{fig:timesfm_scaled_appendix}
\end{figure}

\section{UCI Regression Results}
\label{app:uci-per-dataset}

\subsection{Significance, Coverage, and Accuracy}
\label{app:uci-tables}

Table~\ref{tab:uci_significance} gives the significance of the comparisons in Table~\ref{tab:uci_summary}. A baseline is significantly better than SA in only 2 of the 80 comparisons: MC Dropout on Protein for $W_1$, and HSA on Naval for width, at $0.41$ times SA's width.

\begin{table}[tp]
\centering
\caption{Significance of the comparisons in Table~\ref{tab:uci_summary}, each baseline against SA (Appendix~\ref{app:uci-protocol}). $\dagger$: SA significantly better. $=$: statistically tied. $\triangledown$: baseline significantly better. $^{\ddagger}$Protein's small-sample rule. MC: MC Dropout, SW: SWAG, MS: MultiSWAG, IV: IVON.}
\label{tab:uci_significance}
\small
\setlength{\tabcolsep}{7pt}
\begin{tabular}{@{}l ccccc ccccc@{}}
\toprule
& \multicolumn{5}{c}{\textbf{Panel~A: Native $W_1$}} & \multicolumn{5}{c}{\textbf{Panel~B: Sharp/SA}} \\
\cmidrule(lr){2-6}\cmidrule(l){7-11}
Dataset & MC & SW & MS & HSA & IV & MC & SW & MS & HSA & IV \\
\midrule
Yacht & $=$ & $=$ & $=$ & $=$ & $=$ & $=$ & $=$ & $=$ & $\dagger$ & $\dagger$ \\
Energy & $=$ & $=$ & $=$ & $=$ & $\dagger$ & $\dagger$ & $=$ & $=$ & $\dagger$ & $=$ \\
Naval & $=$ & $\dagger$ & $=$ & $\dagger$ & $\dagger$ & $\dagger$ & $\dagger$ & $\dagger$ & $\triangledown$ & $\dagger$ \\
Power & $\dagger$ & $\dagger$ & $\dagger$ & $\dagger$ & $\dagger$ & $\dagger$ & $\dagger$ & $\dagger$ & $\dagger$ & $\dagger$ \\
Kin8nm & $\dagger$ & $\dagger$ & $\dagger$ & $\dagger$ & $\dagger$ & $\dagger$ & $\dagger$ & $\dagger$ & $\dagger$ & $\dagger$ \\
Concrete & $\dagger$ & $\dagger$ & $\dagger$ & $\dagger$ & $\dagger$ & $\dagger$ & $\dagger$ & $\dagger$ & $\dagger$ & $\dagger$ \\
Protein$^{\ddagger}$ & $\triangledown$ & $=$ & $=$ & $=$ & $\dagger$ & $\dagger$ & $\dagger$ & $\dagger$ & $\dagger$ & $\dagger$ \\
Wine & $=$ & $=$ & $=$ & $=$ & $\dagger$ & $\dagger$ & $\dagger$ & $\dagger$ & $\dagger$ & $\dagger$ \\
\midrule
SA better / tied / worse & \multicolumn{5}{c}{21 / 18 / 1} & \multicolumn{5}{c}{33 / 6 / 1} \\
\bottomrule
\end{tabular}
\end{table}

Table~\ref{tab:uci_cov_acc} reports native coverage and accuracy for all methods. SA's sample mean stays within $15\%$ of the backbone's RMSE on every dataset and has the lowest RMSE of the six methods on Yacht. SA's calibration follows the backbone's accuracy: it covers $0.87$ to $0.89$ with backbone RMSE below $0.2$ of $\sigma_y$, $0.67$ to $0.79$ between $0.2$ and $0.35$, and $0.52$ and $0.36$ on Protein and Wine. On Wine, the fitted $\nu$ equals 1 on all 20 splits, the most dispersed member of the family, so SA cannot widen further there.

\begin{table}[tp]
\centering
\caption{Native coverage and accuracy on UCI, mean with the standard deviation over splits beneath. \textbf{Panel~A}: coverage of the 90\% interval; bold: closest to $0.90$. \textbf{Panel~B}: RMSE$/\sigma_y$ of the deterministic backbone from which SA and MC Dropout draw (BB) and of each method's 99-sample mean; bold: lowest of the six sample means. MC: MC Dropout, SW: SWAG, MS: MultiSWAG, IV: IVON.}
\label{tab:uci_cov_acc}
\def\sm#1{{\fontsize{6}{7}\selectfont #1}}
\footnotesize
\setlength{\tabcolsep}{1.5pt}
\begin{tabular*}{\linewidth}{@{\extracolsep{\fill}}l cccccc ccccccc@{}}
\toprule
& \multicolumn{6}{c}{\textbf{Panel~A: Native coverage}} & \multicolumn{7}{c}{\textbf{Panel~B: RMSE$/\sigma_y\downarrow$}} \\
\cmidrule(lr){2-7}\cmidrule(l){8-14}
Dataset & SA & MC & SW & MS & HSA & IV & BB & SA & MC & SW & MS & HSA & IV \\
\midrule
Yacht & \textbf{0.89} & 0.93 & 0.96 & 0.98 & 0.96 & 0.63 & 0.048 & \textbf{0.050} & 0.053 & 0.059 & 0.058 & 0.065 & 0.055 \\[-2.5pt]
 & \sm{$\pm$0.07} & \sm{$\pm$0.05} & \sm{$\pm$0.04} & \sm{$\pm$0.02} & \sm{$\pm$0.04} & \sm{$\pm$0.11} & \sm{$\pm$0.019} & \sm{$\pm$0.021} & \sm{$\pm$0.022} & \sm{$\pm$0.026} & \sm{$\pm$0.026} & \sm{$\pm$0.029} & \sm{$\pm$0.027} \\[1.5pt]
Energy & \textbf{0.88} & 0.95 & 0.70 & 0.80 & 0.81 & 0.41 & 0.070 & 0.067 & 0.074 & 0.067 & \textbf{0.062} & 0.166 & 0.067 \\[-2.5pt]
 & \sm{$\pm$0.05} & \sm{$\pm$0.03} & \sm{$\pm$0.05} & \sm{$\pm$0.05} & \sm{$\pm$0.05} & \sm{$\pm$0.07} & \sm{$\pm$0.010} & \sm{$\pm$0.009} & \sm{$\pm$0.012} & \sm{$\pm$0.008} & \sm{$\pm$0.007} & \sm{$\pm$0.028} & \sm{$\pm$0.009} \\[1.5pt]
Naval & 0.87 & \textbf{0.90} & 0.79 & 0.84 & 0.78 & 0.30 & 0.182 & 0.169 & 0.185 & 0.586 & 0.532 & \textbf{0.065} & 0.289 \\[-2.5pt]
 & \sm{$\pm$0.04} & \sm{$\pm$0.04} & \sm{$\pm$0.05} & \sm{$\pm$0.02} & \sm{$\pm$0.09} & \sm{$\pm$0.06} & \sm{$\pm$0.025} & \sm{$\pm$0.031} & \sm{$\pm$0.036} & \sm{$\pm$0.082} & \sm{$\pm$0.024} & \sm{$\pm$0.036} & \sm{$\pm$0.039} \\[1.5pt]
Power & \textbf{0.67} & 0.38 & 0.31 & 0.35 & 0.24 & 0.15 & 0.211 & 0.237 & \textbf{0.208} & 0.229 & 0.227 & 0.232 & 0.241 \\[-2.5pt]
 & \sm{$\pm$0.04} & \sm{$\pm$0.05} & \sm{$\pm$0.02} & \sm{$\pm$0.02} & \sm{$\pm$0.02} & \sm{$\pm$0.04} & \sm{$\pm$0.015} & \sm{$\pm$0.015} & \sm{$\pm$0.014} & \sm{$\pm$0.010} & \sm{$\pm$0.010} & \sm{$\pm$0.014} & \sm{$\pm$0.011} \\[1.5pt]
Kin8nm & \textbf{0.76} & 0.59 & 0.61 & 0.69 & 0.68 & 0.23 & 0.255 & 0.277 & 0.254 & 0.261 & \textbf{0.252} & 0.273 & 0.280 \\[-2.5pt]
 & \sm{$\pm$0.02} & \sm{$\pm$0.02} & \sm{$\pm$0.03} & \sm{$\pm$0.02} & \sm{$\pm$0.02} & \sm{$\pm$0.03} & \sm{$\pm$0.009} & \sm{$\pm$0.009} & \sm{$\pm$0.009} & \sm{$\pm$0.009} & \sm{$\pm$0.007} & \sm{$\pm$0.008} & \sm{$\pm$0.010} \\[1.5pt]
Concrete & \textbf{0.79} & 0.66 & 0.47 & 0.56 & 0.67 & 0.13 & 0.323 & 0.335 & 0.321 & 0.319 & \textbf{0.310} & 0.391 & 0.334 \\[-2.5pt]
 & \sm{$\pm$0.05} & \sm{$\pm$0.05} & \sm{$\pm$0.06} & \sm{$\pm$0.05} & \sm{$\pm$0.05} & \sm{$\pm$0.03} & \sm{$\pm$0.031} & \sm{$\pm$0.030} & \sm{$\pm$0.029} & \sm{$\pm$0.031} & \sm{$\pm$0.031} & \sm{$\pm$0.028} & \sm{$\pm$0.030} \\[1.5pt]
Protein & \textbf{0.52} & \textbf{0.52} & 0.43 & 0.47 & 0.41 & 0.11 & 0.586 & 0.653 & \textbf{0.574} & 0.621 & 0.615 & 0.585 & 0.691 \\[-2.5pt]
 & \sm{$\pm$0.05} & \sm{$\pm$0.02} & \sm{$\pm$0.02} & \sm{$\pm$0.01} & \sm{$\pm$0.01} & \sm{$\pm$0.01} & \sm{$\pm$0.005} & \sm{$\pm$0.021} & \sm{$\pm$0.006} & \sm{$\pm$0.011} & \sm{$\pm$0.008} & \sm{$\pm$0.008} & \sm{$\pm$0.010} \\[1.5pt]
Wine & \textbf{0.36} & 0.24 & 0.26 & 0.34 & 0.21 & 0.03 & 0.790 & 0.905 & 0.786 & 0.782 & \textbf{0.770} & 0.785 & 0.790 \\[-2.5pt]
 & \sm{$\pm$0.06} & \sm{$\pm$0.05} & \sm{$\pm$0.05} & \sm{$\pm$0.03} & \sm{$\pm$0.04} & \sm{$\pm$0.02} & \sm{$\pm$0.050} & \sm{$\pm$0.055} & \sm{$\pm$0.047} & \sm{$\pm$0.048} & \sm{$\pm$0.047} & \sm{$\pm$0.047} & \sm{$\pm$0.046} \\[1.5pt]
\bottomrule
\end{tabular*}
\end{table}

The backbone-accuracy effect highlighted in Table~\ref{tab:uci_cov_acc} is further revealed at the scaled operating points. Where the backbone is accurate (Yacht, Energy, Naval), no scaled baseline has a significantly lower $W_1$ than native SA, and scaled coverage lies within $0.05$ of SA's in 14 of 15 comparisons. On the five less accurate datasets (Power, Kin8nm, Concrete, Protein, Wine), the scaled baselines cover $0.89$ to $0.93$ against SA's $0.36$ to $0.79$.

Figure~\ref{fig:uci_pit_grid} shows the PIT histograms natively and after temperature scaling, and Figure~\ref{fig:uci_width_grid} shows the sharpness comparison, for all eight datasets. Using only the first $S$ of 19 predictive samples, for every $S$ from 2 to 19, SA has the lowest pooled $W_1$ on 5 to 7 of the 8 datasets, more than any other method at every $M$.

\begin{figure}[tp]
\centering
\includegraphics{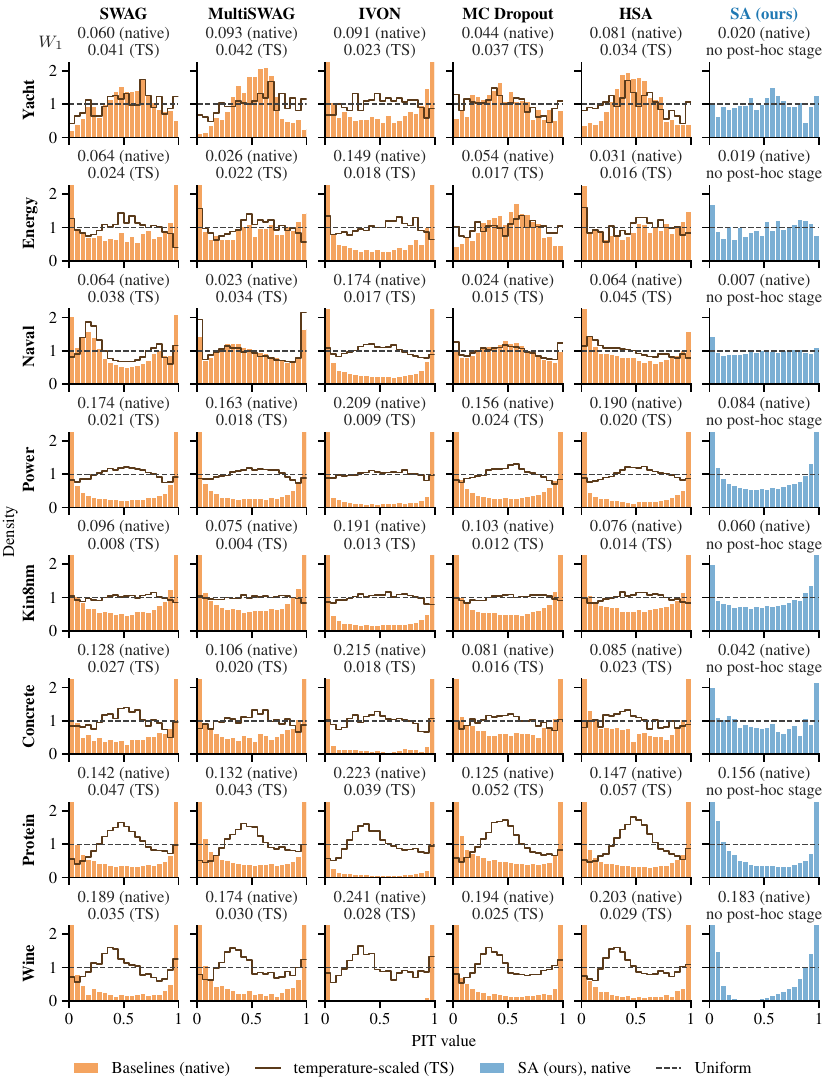}
\caption{PIT on the eight UCI datasets at $S{=}99$, pooled over each dataset's splits, on one density axis (uniform $=1$, dashed), drawn as in Figure~\ref{fig:climax_pit}. Bars are native; the dark step is the baseline after temperature scaling (TS). Above each panel are pooled $W_1$ native and after TS. Native bins that run off the axis reach densities $2.3$--$10.9$ for the baselines and $2.7$--$6.9$ for SA.}
\label{fig:uci_pit_grid}
\end{figure}

\begin{figure}[tp]
\centering
\includegraphics{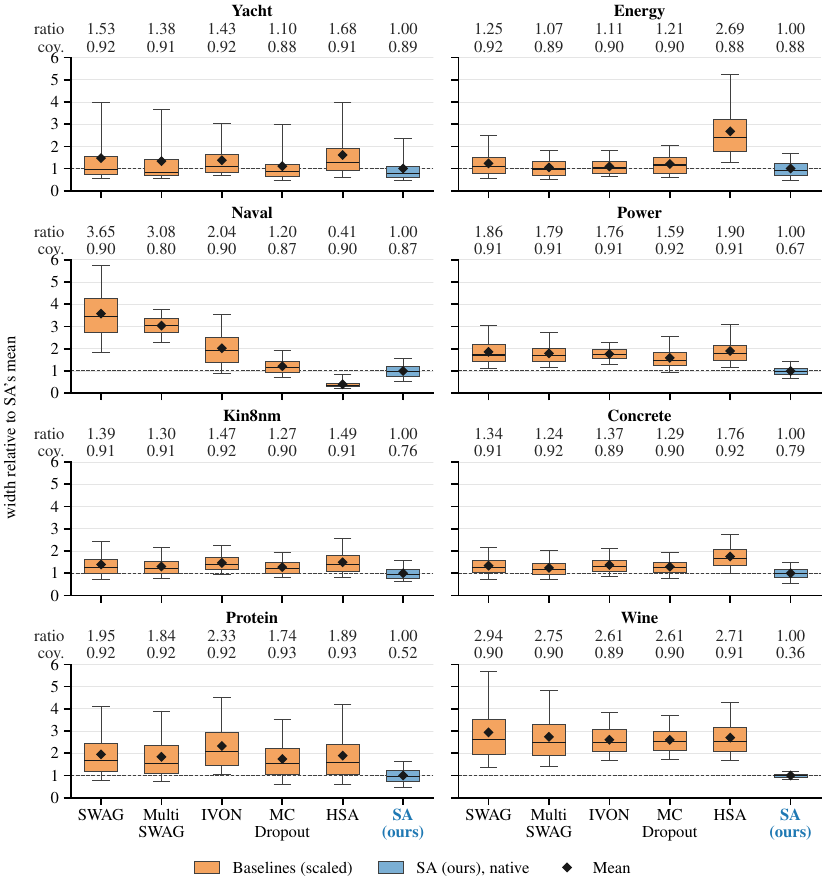}
\caption{Interval widths on the eight UCI datasets at $S{=}99$, over all test points of all splits, relative to SA's mean width on that dataset, with the baselines temperature-scaled and SA native. The interval runs from the 5th to the 95th ordered sample. Boxes are the quartiles, whiskers the 5th and 95th percentiles, and diamonds the means. Above each box are the mean over splits of the width ratio to SA and the coverage of this interval.}
\label{fig:uci_width_grid}
\end{figure}

\subsection{Split Conformal as a Reference}
\label{app:uci-conformal}

Split conformal \cite{Vovk2005,Papadopoulos2002,Lei2018} widens the backbone's point prediction by a residual quantile from the validation fold, which gives one interval at one chosen level and no predictive distribution. Table~\ref{tab:uci_conformal} compares it at the 90\% level with SA's native interval. SA's interval is the sharper on all eight datasets, with conformal's $1.05$ to $2.50$ times as wide, and SA supplies a full predictive distribution. Conformal covers $0.881$ to $0.910$ against SA's $0.363$ to $0.885$, and it has the lower interval score \cite{GneitingRaftery2007} on five of the eight datasets, with SA lower on Yacht, Energy, and Naval, where the backbone is accurate. We therefore treat conformal as a complement for deployments that must meet a stated coverage, not as a substitute for SA's sharper predictive distribution.

\begin{table}[tp]
\centering
\caption{Split conformal on the deterministic backbone against SA's native interval at the 90\% level, mean over splits, in the target's original units. IS: interval score, with the lower of the two in bold. Ratio: conformal's width over SA's.}
\label{tab:uci_conformal}
\footnotesize
\setlength{\tabcolsep}{4pt}
\begin{tabular}{@{}l ccc ccc c@{}}
\toprule
& \multicolumn{3}{c}{Conformal} & \multicolumn{3}{c}{SA} & \\
\cmidrule(lr){2-4}\cmidrule(lr){5-7}
Dataset & Coverage & Width & IS & Coverage & Width & IS & Ratio \\
\midrule
Yacht & 0.910 & 2.19 & 3.85 & 0.885 & 1.53 & \textbf{2.93} & 1.43 \\
Energy & 0.881 & 2.16 & 3.26 & 0.876 & 1.90 & \textbf{2.89} & 1.14 \\
Naval & 0.895 & 0.00869 & 0.0112 & 0.872 & 0.00828 & \textbf{0.0110} & 1.05 \\
Power & 0.897 & 10.6 & \textbf{15.7} & 0.666 & 7.79 & 21.5 & 1.36 \\
Kin8nm & 0.898 & 0.218 & \textbf{0.291} & 0.764 & 0.177 & 0.347 & 1.23 \\
Concrete & 0.889 & 16.4 & \textbf{25.0} & 0.790 & 13.7 & 27.0 & 1.20 \\
Protein & 0.895 & 11.7 & \textbf{17.9} & 0.516 & 6.65 & 22.7 & 1.76 \\
Wine & 0.890 & 2.08 & \textbf{2.78} & 0.363 & 0.832 & 5.64 & 2.50 \\
\bottomrule
\end{tabular}
\end{table}

\begin{table}[tp]
\centering
\caption{Out-of-distribution behavior on UCI ($S{=}19$; $W_1$ with 20 rank bins; coverage of the $[\min,\max]$ range, expected $0.90$ under calibration); the last two columns are the mean and the worst case over the datasets. \textit{Inside hull}: share of test points of the standard splits inside the convex hull of their training set, in the principal subspace retaining 99\% of training variance. \textit{Held-out regions}: the test region lies beyond the training and validation range of one feature, mean over the 3 repetitions of each axis and then over axes, with $\nu$ and temperatures fitted on interior data. Bold: best; underlined: second best, per column.}
\label{tab:ood_central}
\footnotesize
\setlength{\tabcolsep}{3pt}
\begin{tabular}{@{}lcccccccccc@{}}
\toprule
Method & Yacht & Energy & Naval & Power & Kin8nm & Concrete & Protein & Wine & mean & worst \\
\midrule
Inside hull (\%) & 83.2 & 82.3 & 99.6 & 97.4 & 51.5 & 63.5 & 98.0 & 46.0 & & \\
\midrule
\multicolumn{11}{@{}l}{Held-out regions: native $W_1\downarrow$} \\
SWAG & 0.212 & 0.475 & \textbf{0.290} & \underline{0.172} & 0.121 & 0.211 & \textbf{0.172} & \textbf{0.186} & 0.230 & 0.475 \\
IVON & 0.282 & 0.452 & 0.377 & 0.199 & 0.194 & 0.254 & 0.227 & 0.234 & 0.277 & 0.452 \\
MC Dropout & 0.221 & \underline{0.426} & \underline{0.315} & \textbf{0.165} & 0.118 & 0.174 & 0.190 & \underline{0.190} & 0.225 & \underline{0.426} \\
HSA & \textbf{0.208} & \textbf{0.341} & 0.346 & 0.180 & \underline{0.103} & \underline{0.157} & 0.197 & 0.194 & \textbf{0.216} & \textbf{0.346} \\
SA (ours) & \underline{0.210} & 0.466 & 0.339 & 0.180 & \textbf{0.084} & \textbf{0.134} & \underline{0.179} & 0.203 & \underline{0.224} & 0.466 \\
\midrule
\multicolumn{11}{@{}l}{Held-out regions: $W_1\downarrow$, baselines temperature-scaled} \\
SWAG & \textbf{0.206} & 0.467 & \textbf{0.286} & \textbf{0.049} & \underline{0.054} & 0.125 & \textbf{0.067} & \underline{0.060} & \underline{0.164} & 0.467 \\
IVON & 0.212 & \underline{0.433} & 0.346 & 0.066 & 0.059 & 0.126 & \underline{0.072} & 0.066 & 0.172 & \underline{0.433} \\
MC Dropout & 0.219 & 0.438 & \underline{0.316} & \underline{0.055} & \textbf{0.047} & \underline{0.120} & 0.079 & \textbf{0.057} & 0.166 & 0.438 \\
HSA & 0.214 & \textbf{0.333} & 0.318 & 0.058 & 0.056 & \textbf{0.114} & 0.080 & 0.068 & \textbf{0.155} & \textbf{0.333} \\
SA (ours) & \underline{0.210} & 0.466 & 0.339 & 0.180 & 0.084 & 0.134 & 0.179 & 0.203 & 0.224 & 0.466 \\
\midrule
\multicolumn{11}{@{}l}{Held-out regions: coverage, baselines temperature-scaled} \\
SWAG & \textbf{0.71} & 0.12 & \textbf{0.69} & \textbf{0.90} & \textbf{0.90} & 0.79 & 0.92 & 0.92 & \underline{0.75} & 0.12 \\
IVON & \underline{0.68} & 0.19 & 0.29 & \underline{0.91} & 0.87 & 0.77 & \underline{0.89} & \underline{0.91} & 0.69 & 0.19 \\
MC Dropout & \underline{0.68} & \underline{0.20} & 0.33 & \underline{0.91} & \underline{0.89} & \underline{0.83} & \textbf{0.90} & \textbf{0.90} & 0.70 & \underline{0.20} \\
HSA & \textbf{0.71} & \textbf{0.62} & \underline{0.59} & 0.94 & \underline{0.89} & \textbf{0.90} & 0.93 & 0.94 & \textbf{0.81} & \textbf{0.59} \\
SA (ours) & 0.66 & 0.12 & 0.27 & 0.59 & 0.72 & 0.70 & 0.46 & 0.42 & 0.49 & 0.12 \\
\midrule
\multicolumn{11}{@{}l}{Held-out regions: width relative to SA$\downarrow$, baselines temperature-scaled} \\
SWAG & \underline{1.29} & 1.42 & 3.41 & \underline{1.73} & 1.52 & 1.35 & 2.26 & 3.32 & 2.04 & 3.41 \\
IVON & 1.35 & \textbf{0.88} & 1.81 & \underline{1.73} & 1.55 & \underline{1.21} & 2.16 & \underline{2.80} & \underline{1.69} & \underline{2.80} \\
MC Dropout & 1.91 & 1.22 & \underline{1.52} & 1.77 & \underline{1.41} & 1.34 & \underline{2.14} & 2.84 & 1.77 & 2.84 \\
HSA & 2.48 & 3.21 & 1.59 & 2.16 & 1.65 & 1.94 & 2.74 & 3.43 & 2.40 & 3.43 \\
SA (ours) & \textbf{1.00} & \underline{1.00} & \textbf{1.00} & \textbf{1.00} & \textbf{1.00} & \textbf{1.00} & \textbf{1.00} & \textbf{1.00} & \textbf{1.00} & \textbf{1.00} \\
\bottomrule
\end{tabular}
\end{table}

\subsection{Out-of-Distribution Behavior}
\label{app:uci-ood}

Table~\ref{tab:ood_central} collects the out-of-distribution results. Random splits are largely interpolation: in the principal subspace that retains 99\% of training variance, $46.0\%$ to $99.6\%$ of test points lie inside the convex hull of their training set, although membership depends on the coordinates \cite{Balestriero2021}, and on Naval, whose inputs are nearly two-dimensional, the share falls to $1.3\%$ in all 14 features.

To force extrapolation, we hold out the top of one feature at a time, about 10\% of the rows, which gives 61 axes and, with three random divisions of the interior into training and validation, 183 conditions. Every method is retrained on each condition, and every calibration parameter is fitted on interior data only. Every single-model method's $W_1$ is worse than on the standard splits on all eight datasets, SA's included. SA has the strictly lowest native $W_1$ most often, on 23 of the 61 axes against 15 for the next method, a descriptive count with small margins. Temperatures fitted on interior data restore baseline coverage to $0.77$ to $0.94$ on five datasets, against SA's $0.42$ to $0.72$, at $1.2$ to $3.4$ times SA's width, and leave SA lowest on 2 of the 61 axes. MultiSWAG, at ten trainings per condition, runs on 24 of the conditions: there it has the lowest mean $W_1$, $0.168$ against SA's $0.205$, mainly through Naval, where it alone stays near nominal coverage ($0.91$). Independent trainings carry model uncertainty that no single-model mechanism reproduces there, and SA composes with them.

\section{Proper Scores and Operating Points}
\label{app:proper-scores}

This appendix collects the proper-score evaluations: SA's CRPS and Energy Score as functions of $\nu$ on ClimaX (Appendix~\ref{app:nu-ablation}), the decomposition of every method's CRPS on ClimaX (Appendix~\ref{app:crps-decomp}), and the ensemble CRPS on UCI (Appendix~\ref{app:uci-crps}).

\subsection{Score-Oriented $\nu$ Ablation}
\label{app:nu-ablation}

Figure~\ref{fig:nu_ablation} varies $\nu$ on ClimaX with everything else as in the $\nu^\star{=}9$ runs (single backbone, calibration year 2016, all 17{,}376 test forecasts, $S{=}19$). CRPS and Energy Score are area-weighted and use the fair ensemble estimator~\cite{Ferro2014}, which removes the finite-ensemble bias that favors narrow ensembles; the Energy Score is that of each forecast's whole field, in the area-weighted root-mean-square norm. Both scores are lowest at $\nu{=}25$, where CRPS is $131.7\,\mathrm{m^2\,s^{-2}}$, $11\%$ below its value at $\nu^\star$ ($147.9$), but the PIT there is markedly U-shaped ($W_1{=}0.093$ against $0.040$ at $\nu^\star$). Proper scores and calibration thus favor different members of the same one-parameter family, and $\nu$ moves SA between them without retraining. The mean of the 19 samples approaches the deterministic prediction as $\nu$ grows (Proposition~\ref{prop:unbiased}), and its RMSE falls from $323.4$ at $\nu^\star$ to $268.3\,\mathrm{m^2\,s^{-2}}$ at $\nu{=}100$.

\begin{figure}[tp]
\centering
\includegraphics{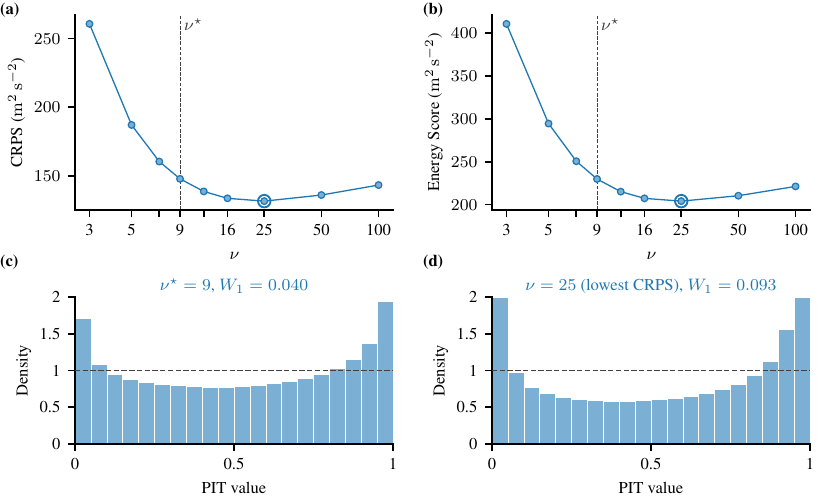}
\caption{Proper scores against $\nu$ on ClimaX ($Z_{500}$, 72-hour lead, 17{,}376 test forecasts, $S{=}19$), area-weighted. \textbf{(a)} CRPS and \textbf{(b)} Energy Score, both with the fair estimator. The circle marks the lowest value and the dashed line marks $\nu^\star{=}9$. At $\nu{=}1$, which is off the axes, the CRPS is $776.7$ and the Energy Score $1022.9$. \textbf{(c)} The PIT at $\nu^\star{=}9$ and \textbf{(d)} at $\nu{=}25$, the value with the lowest CRPS; the dashed line is the uniform density. At $\nu{=}25$ the two outermost bins reach $2.2$ and $4.0$, above the axis.}
\label{fig:nu_ablation}
\end{figure}

\subsection{CRPS Decomposition on ClimaX}
\label{app:crps-decomp}

Table~\ref{tab:climax_crps_decomp} splits each method's CRPS on ClimaX into the prediction error $\mathbb{E}|X{-}y|$ of its samples and the spread reward $\tfrac{1}{2}\mathbb{E}|X{-}X'|$. SWAG and MultiSWAG have the smallest prediction error but almost no spread reward ($5.4$ and $5.8\,\mathrm{m^2s^{-2}}$): their native ensembles are nearly deterministic, so their CRPS is close to the absolute error of a point forecast. SA at $\nu^\star$ has the largest spread reward ($103.1$) and the second-lowest CRPS ($153.6$, after Contextual Dropout's $139.9$). Its prediction error is that of individual samples, which SA spreads around the deterministic prediction, so raising $\nu$ trades spread for prediction error (Appendix~\ref{app:nu-ablation}).

\begin{table}[tp]
\caption{ClimaX CRPS decomposition (native, $T{=}1$; single backbone, calibration year 2016, $S{=}19$; area-weighted) with the standard ensemble estimator, $\mathrm{CRPS}=\mathbb{E}|X{-}y|-\tfrac{1}{2}\mathbb{E}|X{-}X'|$. $\mathbb{E}|X{-}y|$ is the prediction error (lower is better) and $\tfrac{1}{2}\mathbb{E}|X{-}X'|$ the spread reward (higher is more diverse); Appendix~\ref{app:nu-ablation} reports SA's CRPS with the fair estimator. SA is at $\nu^\star{=}9$. Units: $\mathrm{m}^2\,\mathrm{s}^{-2}$. Bold: best; underlined: second best.}
\label{tab:climax_crps_decomp}
\centering
\small
\begin{tabular}{@{}lccc@{}}
\toprule
Method & CRPS $\downarrow$ & $\mathbb{E}|X{-}y|$ $\downarrow$ & $\tfrac{1}{2}\mathbb{E}|X{-}X'|$ $\uparrow$ \\
\midrule
SWAG                        & 166.77 & \underline{172.18} & 5.41 \\
MultiSWAG                   & 166.06 & \textbf{171.88} & 5.82 \\
IVON                        & 241.43 & 257.11 & 15.69 \\
Contextual Dropout          & \textbf{139.94} & 201.35 & \underline{61.41} \\
HSA                         & 377.34 & 427.16 & 49.82 \\
\midrule
SA (ours, $\nu^\star{=}9$)  & \underline{153.64} & 256.68 & \textbf{103.05} \\
\bottomrule
\end{tabular}
\end{table}

\subsection{UCI Ensemble CRPS}
\label{app:uci-crps}

We evaluate ensemble CRPS natively for all six methods at $S{=}99$, with the fair estimator of \citet{Ferro2014} used on ClimaX (Appendix~\ref{app:nu-ablation}); Table~\ref{tab:uci_crps} reports mean$\pm$std over splits. SA ranks among the two best methods on 4 of 8 datasets (Yacht, Energy, Naval, Power). MultiSWAG ranks among the two best on 4 of 8 datasets (Energy, Kin8nm, Concrete, Wine), where averaging across 10 independently trained components reduces per-sample prediction error. These rankings are at the calibration-first $\nu$; on ClimaX, moving $\nu$ along the same family lowers both proper scores at a cost in $W_1$ (Appendix~\ref{app:nu-ablation}), and we did not sweep $\nu$ on UCI.

\begin{table}[tp]
\caption{Ensemble CRPS on UCI regression ($S{=}99$ samples, fair estimator, native predictions, mean$\pm$std over splits; 20 per dataset, 5 for Protein). Lower is better; \textbf{bold} marks the two best per row.}
\label{tab:uci_crps}
\centering
\small
\setlength{\tabcolsep}{4pt}
\renewcommand{\arraystretch}{1.1}
\resizebox{\textwidth}{!}{%
\begin{tabular}{@{}lcccccc@{}}
\toprule
Dataset & SA (ours) & MC Dropout & SWAG & MultiSWAG & HSA & IVON \\
\midrule
Yacht & $\mathbf{0.316{\pm}0.110}$ & $\mathbf{0.330{\pm}0.113}$ & $0.384{\pm}0.123$ & $0.358{\pm}0.107$ & $0.436{\pm}0.157$ & $0.402{\pm}0.166$ \\
Energy & $\mathbf{0.354{\pm}0.050}$ & $0.382{\pm}0.048$ & $0.361{\pm}0.044$ & $\mathbf{0.326{\pm}0.035}$ & $0.899{\pm}0.161$ & $0.413{\pm}0.062$ \\
Naval & $\mathbf{0.0014{\pm}0.0003}$ & $0.0016{\pm}0.0003$ & $0.0051{\pm}0.0009$ & $0.0042{\pm}0.0002$ & $\mathbf{0.0006{\pm}0.0003}$ & $0.0029{\pm}0.0005$ \\
Power & $\mathbf{2.291{\pm}0.175}$ & $\mathbf{2.122{\pm}0.155}$ & $2.550{\pm}0.072$ & $2.480{\pm}0.070$ & $2.675{\pm}0.195$ & $2.914{\pm}0.145$ \\
Kin8nm & $0.0411{\pm}0.0014$ & $\mathbf{0.0397{\pm}0.0014}$ & $0.0405{\pm}0.0015$ & $\mathbf{0.0380{\pm}0.0011}$ & $0.0415{\pm}0.0014$ & $0.0509{\pm}0.0021$ \\
Concrete & $3.070{\pm}0.330$ & $\mathbf{3.012{\pm}0.302}$ & $3.136{\pm}0.266$ & $\mathbf{2.927{\pm}0.269}$ & $3.810{\pm}0.306$ & $3.912{\pm}0.333$ \\
Protein & $2.309{\pm}0.148$ & $\mathbf{1.907{\pm}0.051}$ & $2.165{\pm}0.056$ & $2.104{\pm}0.033$ & $\mathbf{2.008{\pm}0.032}$ & $2.947{\pm}0.082$ \\
Wine & $0.477{\pm}0.032$ & $0.436{\pm}0.027$ & $\mathbf{0.428{\pm}0.032}$ & $\mathbf{0.408{\pm}0.027}$ & $0.441{\pm}0.030$ & $0.487{\pm}0.033$ \\
\bottomrule
\end{tabular}%
}
\end{table}

\clearpage

\end{document}